\documentclass[letterpaper]{article} % DO NOT CHANGE THIS
\usepackage{aaai25}  % DO NOT CHANGE THIS
\usepackage{algorithm}
\usepackage{algorithmic}
\usepackage{amsmath}
\usepackage{amsthm}
\usepackage{amsfonts}
\usepackage[capitalise]{cleveref}
\usepackage{times}  % DO NOT CHANGE THIS
\usepackage{helvet}  % DO NOT CHANGE THIS
\usepackage{courier}  % DO NOT CHANGE THIS
\usepackage[hyphens]{url}  % DO NOT CHANGE THIS
\usepackage{graphicx} % DO NOT CHANGE THIS
\usepackage{natbib}  % DO NOT CHANGE THIS AND DO NOT ADD ANY OPTIONS TO IT
\usepackage{caption} % DO NOT CHANGE THIS AND DO NOT ADD ANY OPTIONS TO IT
\newtheorem{thm}{Theorem}

\newtheorem{cor}{Corollary}
\newtheorem{lem}{Lemma}

\newtheorem{assumption}{Assumption}
\newtheorem{definition}{Definition}
\newtheorem{condition}{Condition}
\ifodd 0
\newcommand{\rev}[1]{{\color{blue}#1}}
\else
\newcommand{\rev}[1]{#1}
\fi
\usepackage{todonotes}

\ifodd 1

\else

\fi
\newcommand{\pluseq}{\mathrel{+}=}
\usepackage{thmtools}
\usepackage{thm-restate}

\usepackage{amsmath}
\DeclareMathOperator*{\argmax}{arg\,max}

\usepackage{newfloat}
\usepackage{listings}
\DeclareCaptionStyle{ruled}{labelfont=normalfont,labelsep=colon,strut=off} % DO NOT CHANGE THIS
\floatstyle{ruled}
\newfloat{listing}{tb}{lst}{}
\floatname{listing}{Listing}
\nocopyright

\title{Neural Combinatorial Clustered Bandits for Recommendation Systems}
\author{
    Baran Atalar\textsuperscript{\rm 1}, Carlee Joe-Wong\textsuperscript{\rm 1}\\
}
\affiliations{
    \textsuperscript{\rm 1}Carnegie Mellon University\\
    batalar@andrew.cmu.edu, cjoewong@andrew.cmu.edu
}
\usepackage{bibentry}
\begin{document}

\maketitle

\begin{abstract}
We consider the contextual combinatorial bandit setting where in each round, the learning agent, e.g., a recommender system, selects a subset of ``arms,'' e.g., products, and observes rewards for both the individual base arms, which are a function of known features (called ``context''), and the super arm (the subset of arms), which is a function of the base arm rewards. % The expected reward of an arm is a function of its (known) features (context), whereas the reward of a super arm is a function of its base arms' rewards. 
  The agent's goal is to simultaneously learn the unknown reward functions and choose the highest-reward arms. For example, the ``reward'' may represent a user's probability of clicking on one of the recommended products. Conventional bandit models, however, employ restrictive reward function models in order to obtain performance guarantees. We make use of deep neural networks to estimate and learn the unknown reward functions and propose Neural UCB Clustering (NeUClust), which adopts a clustering approach to select the super arm in every round by exploiting underlying structure in the context space. %There is often an underlying structure in the context space for applications such as recommender systems. 
  Unlike prior neural bandit works, NeUClust uses a neural network to estimate the super arm reward and select the super arm, thus eliminating the need for a known optimization oracle. We non-trivially extend prior neural combinatorial bandit works to prove that NeUClust achieves $\widetilde{O}\left(\widetilde{d}\sqrt{T}\right)$ regret, where $\widetilde{d}$ is the effective dimension of a neural tangent kernel matrix, $T$ the number of rounds. Experiments on real world recommendation datasets show that NeUClust achieves better regret and reward than other contextual combinatorial and neural bandit algorithms.
\end{abstract}


\section{Introduction}
The contextual combinatorial bandit (CC-MAB) setting has been applied to content recommendation such as movies \cite{Qin2014ContextualCB}, where we recommend multiple movies to a user or select multiple users to whom a movie will be recommended. %\carlee{Since we focus on recommender systems, we could take out crowdsourcing (done)}%The context in these applications are related to the quality and location of the workers for crowdsourcing and ratings of users and genres of the movies for the movie recommendation. 
%Some other applications include restaurant recommendation and network routing problems which has been explored using cascading feedback \cite{pmlr-v48-lif16}. 
In this application, an agent is confronted with a set of choices, or ``arms,'' e.g., a set of movies, each of which will yield an unknown reward that depends on a known context, e.g.,  movie genre. Over the course of several rounds, the agent must learn to select the ``best'' arms.
%
%In this work, we consider a contextual combinatorial multi-armed bandit (CC-MAB) setting where in each round, the agent selects a subset of the available base arms based on the context of the arms available and the past rewards of the arms played. 
We consider the usual semi-bandit feedback setting where in each round, the user receives the individual (base arm) rewards of the arms selected, which are functions of their contexts, as well as a total (super arm) reward, which is a function of the base arm rewards. %The base arm reward is a function of the context. 
However, both the base arm and the super arm reward functions are stochastic and unknown to the user, and we aim to learn these reward functions in an online manner in order to select the arms which would yield the highest rewards for a given context. \rev{The main challenge in the CC-MAB and bandits setting in general is to \textit{balance between exploiting arms that we already know to be good and exploring new arms} that might be better. A relatively new approach to the exploration-exploitation tradeoff has been to use neural networks to learn the base arm reward functions}.

\textbf{Drawbacks of existing formulations.} Modeling reward functions with neural networks introduces greater expressibility into the rewards learned, but existing neural and non-neural combinatorial bandit works generally require the existence of an exact or approximation oracle \cite{pmlr-v202-hwang23a,elahi2023contextual,pmlr-v28-chen13a,chen2018ccmab}, which selects the super arm after being given as input the estimated reward/quality of each base arm. This oracle, however, may be difficult to derive a priori, before the super arm reward function is known. Indeed, many combinatorial bandit problems are NP-hard and may not even have known approximation algorithms~\cite{kong2021hardness}. We \textbf{eliminate the need for an oracle} by leveraging the structure of recommendation problems. In particular, we recognize that \textit{the context space of many recommendation problems has a clustered structure}. Thus, we can cluster arms by clustering their contexts and use the resulting groupings, aided by an additional neural network-based estimation of super arm rewards, to guide our search for the optimal super arm. %For example for movie recommendation, users that like horror or thriller movies and users who mostly prefer action and adventure movies could be two different clusters. Thus, if we can identify the cluster that best matches the user (i.e., has highest rewards), we can simply select base arms from this cluster. 
% \carlee{This might be confusing, since we cluster the contexts, and it isn't clear how they correspond to the arms, or that the arms are the users (removed it for now)}

\textbf{Challenges of our approach.} To estimate the performance of the clusters, we need an expression for the super arm reward for the base arms in each cluster. We do so by again using a neural network to model the super arm reward and taking advantage of the recommendation problem structure: super arm reward functions in recommendation settings are often \textit{monotonic}, since users are more likely to prefer a group of recommended products if they like each individual product. Thus, we can utilize a monotonic neural network to estimate the super arm reward. One technical challenge is incorporating this structure into a theoretical regret bound: by using a second neural network for the super arm selection, we cannot use existing proof approaches that use an oracle~\cite{pmlr-v202-hwang23a}. Another challenge is accounting for the effect of the error in clustering the context space on the regret bound. %\carlee{Another challenge might be how to incorporate clustering error into the regret bound}

\subsection{Related Work}
 There have been some works in the bandit literature adaptively cluster users (arms) such as \cite{pmlr-v32-gentile14}, \cite{bui2012clustered} but their problem setting is not of a combinatorial nature and thus cannot be directly applied to our setting. Most commonly, prior works %There are works in the literature which 
 make certain \textbf{parametric assumptions} about the model to be learned such as linearity with respect to the context \cite{10.5555/3045118.3045237,cascading_zong,linucb}. 
 % Such works extend and improve the works done in the contextual bandit setting \cite{linucb} to the combinatorial setting by making use of the assumptions about the structure of the learning model. 
% However, such works can perform poorly when trying to learn more complex models that violate the linearity assumption.
%
% Aside from assuming a certain structure on the model to be learned, there are also 
Towards more general reward models, some \textbf{discretization based approaches} use fixed or adaptive discretization \cite{chen2018ccmab,pmlr-v108-nika20a} of the context space to exploit similarity of context features. However, such methods could scale poorly to large context spaces, and the difficulty of the problem highly depends on the complexity of the super arm reward model. 

More recently, there has been a growing literature of using \textbf{neural networks} to facilitate learning in the bandit setting. These advancements exploit recent results on the generalization of deep neural networks \cite{ntkjacot,arora2019,cao2019} and require no parametric assumptions on the reward as a function of the context other than an upper bound \cite{neuralucbzhou20a,zhang2021neural}. % The new results in the generalizability of neural networks has allowed the application of neural networks into 
% Neural networks have been applied to contextual bandits for both upper confidence bound (UCB) and Thompson Sampling (TS) algorithms %have been proposed for the neural contextual bandits 
% \cite{neuralucbzhou20a,zhang2021neural}. 
Yet while the field of neural contextual bandits is relatively well explored, to our knowledge, there is only one neural contextual combinatorial bandit study in the literature with regret bounds \cite{pmlr-v202-hwang23a}. However, unlike our work they only use one neural network to learn the base arm reward function and rely solely on an exact oracle to do super arm selections. % there are very few applications of neural networks to the contextual combinatorial setting. 
\cite{multifacet} chooses multiple arms in each round, but % is similar to the contextual combinatorial bandit setting in that every round multiple arms are chosen, however 
their setting is quite different since they assume $K$ bandits and select an arm from each bandit in every round. 
% To our knowledge, there is only one neural contextual combinatorial bandit study in the literature with regret bounds \cite{pmlr-v202-hwang23a}. % They propose UCB and TS algorithms for the combinatorial setting; 
 %\carlee{Move these last sentences on monotonicity into the second paragraph of the intro that explains how our work solves shortcomings of prior work} \rev{Even though neural bandits yield good empirical results and make fewer assumptions, there are also some drawbacks to using neural networks to learn the reward function/s such as the possibility of inefficiency and longer algorithm runtimes due to a large number of parameters in the network. This issue is further enhanced by the fact that no assumptions are made on the structure of the network, which makes the search space large and extends the training.}

\subsection{Our Contributions}

In this work, we propose a \textbf{provably efficient contextual combinatorial neural bandit algorithm} while making mild assumptions about the reward functions. % \rev{such as continuity and monotonicity}. 
Our proposed algorithm makes use of two neural networks to learn the base arm and super arm reward functions, as shown in Figure~\ref{fig:motivation}.
% and the structure of the super arm network ensures that it is monotonic while also maximizing the cumulative super arm reward. \rev{Assuming monotonicity of the super arm reward function is common in combinatorial settings and in this work, we exploit this structure to accelerate the neural network training and eliminate the need for a known optimization oracle.} \carlee{maybe can take out details of monotonicity to save space here} 
%Figure~\ref{fig:motivation} shows an overview of our setting. 
We formulate our problem in the Problem Formulation Section and then make the following \textbf{contributions:}
%
%Our main contributions are as follows: 
\begin{itemize}
    \item We propose a \textbf{neural contextual combinatorial clustered bandit algorithm} (NeUClust) in the Proposed NeUClust Algorithm Section. To the best of our knowledge, it is the first to use neural networks to learn the base arm and super arm reward functions while making use of clustering to exploit the underlying structure of the context space to guide super arm selections.
    \item We \textbf{prove} (in Regret Analysis and Proofs) that our proposed algorithm achieves $\widetilde{O}\left(\widetilde{d}\sqrt{T}\right)$ regret over $T$ rounds, where $\widetilde{d}$ is the effective dimension of a neural tangent kernel matrix.
    \item We show through \textbf{experimental results} on real world recommendation datasets (MovieLens, Yelp) that our algorithm outperforms state-of-the-art neural contextual/combinatorial and CC-MAB algorithms (Experiments).
\end{itemize}

\section{Problem Formulation}\label{sec:formulation}

We first introduce some useful mathematical notation before outlining our CC-MAB model, the neural network models we employ for the reward functions, and our clustering-based method for selecting arms in each round.
%
%\carlee{We should have a diagram showing the base arms and superarm, and the online feedback loop from the bandit formulation}
\begin{figure}[t]
\centering
\includegraphics[width=\linewidth]{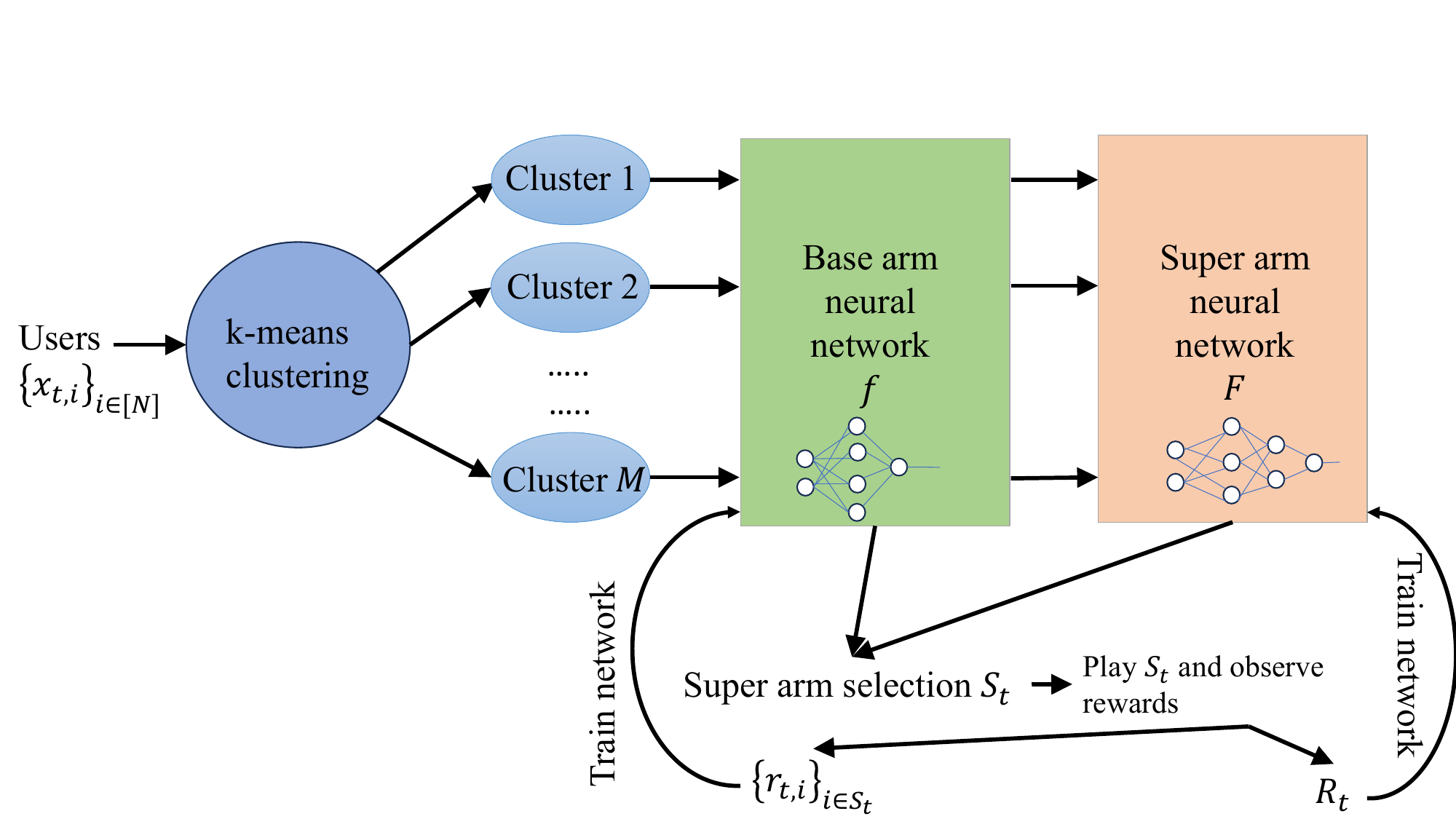}
\caption{Our neural contextual combinatorial bandit formulation, with the online feedback loop for arm selection.}
\label{fig:motivation}
\end{figure}
For a vector $\mathbf{x}\in\mathbb{R}^d$, we represent its $\ell_2$ norm by $\|\mathbf{x}\|_2$, $\ell_1$ norm by $\|\mathbf{x}\|_1$, $\ell_0$ norm by $\|\mathbf{x}\|_0$, and transpose by $\mathbf{x}^T$. The $\ell_2$ norm by a positive definite matrix $\mathbf{A}$ is defined by $\|\mathbf{x}\|_\mathbf{A} := \sqrt{\mathbf{x}^T\mathbf{A}\mathbf{x}}$. We denote by $[N]$, $N \in \mathbb{Z}^+$, the set $\left\{1,2,\ldots,N\right\}$. We let $\mathbf{1}^K = (1, 1, ..., 1) \in \mathbb{R}^K$ denote the ones vector.
% of positive integers up to and including $N$.

\subsection{Contextual Combinatorial Bandit Setting}
In this paper, we consider a contextual combinatorial bandit with $N$ being the total number of \textbf{arms} and $T$ the total number of \textbf{rounds}. At the start of every round $t$, the agent observes the \textbf{context} vectors of all arms, denoted by $\{\mathbf{x}_{t,i} \in \mathbb{R}^d$ $ | $ $i \in [N] \}$, and chooses a subset of the available arms, namely a \textbf{super arm} $S_t \subset [N]$ with a fixed size, or budget, $K$, i.e., $|S_t|=K$. We denote by $\mathcal{S}$ the set of all feasible super arms with cardinality $K$, i.e., $\mathcal{S} = \{ S \subset [N] $ $ |  $ $ |S|=K\}$. In the application of movie recommendation, for example, the base arms could correspond to users where there is an incoming movie at every round which we would like to recommend to a subset of users. The context would correspond to past information about the ratings the users gave to movies and the genres of the movies the users rated.
%\carlee{Introduce a recommendation system example to illustrate the formulation: e.g., what are the arms, contexts, etc. in such a setting, why is semi-bandit feedback reasonable}

When a super arm $S_t \in \mathcal{S}$ is selected in round $t$, the agent observes the \textbf{base arm rewards} of the chosen super arm, namely $\{ r_{t,i}\}_{i \in S_t}$ and receives a \textbf{total (super arm) reward} of $R(S_t,\mathbf{r}_t)$ where $\mathbf{r}_t = [r_{t,i}]_{i \in S_t}$ hence the super arm reward is a function of the individual base arm rewards. This type of bandit feedback at every round is often referred to as \textbf{semi-bandit feedback}, which is common in combinatorial bandits \cite{audibert2013regret}. In the movie recommendation example setup, the base arm and super arm rewards would correspond to the rating given by one user and the collection of users we selected, which would be information that would be available to us. We use these observed super arm and base arm rewards in every round to train our neural networks as explained in more detail in the Proposed NeUClust Algorithm Section. We assume that the base arm rewards $r_{t,i}$ are generated as follows:
\begin{align}
r_{t,i} = h(\mathbf{x}_{t,i}) + \xi_{t,i}
\label{eq:reward}
\end{align}
$ \forall t \in [T]$ and $i\in[N]$ where $h$ is an unknown function satisfying $0 \leq h(\mathbf{x}) \leq 1$ and $\xi_{t,i}$ is $\zeta$-sub-Gaussian noise satisfying $\mathbb{E}[\xi_{t,i}|\mathbf{x}_{1,S_1},...,\mathbf{x}_{t-1,S_{t-1}}]=0$ where $S_t \in \mathcal{S}$  denotes the selected super arm at round $t$. This is a standard assumption for the stochastic bandit setting. \rev{%In every round of movie recommendation, there is an incoming movie and we try to select users to recommend the movie to, so an arm corresponds to a user and the reward for that arm corresponds to the rating given by the user to the movie. 
In our movie recommendation setting, (\ref{eq:reward}) models the fact that the rating that a user gives to a movie mainly depends on the user's genre preferences, which have some uncertainty associated with it %. In other words, even for movies with similar genres, a user's ratings may vary since the genres of the movie are 
 as movie genre is not the sole factor influencing how much a user likes a movie.}

\subsection{Base Arm Reward Function}
To learn the base arm reward function $h$, we use a fully connected neural network with depth $L+1\geq 3$ defined as follows \cite{neuralucbzhou20a,zhang2021neural}:
\begin{align}
f(\mathbf{x};\boldsymbol{\theta}) = \sqrt{m}\mathbf{W}_{L}\sigma(\mathbf{W}_{L-1}\sigma(...\sigma(\mathbf{W}_0\mathbf{x}))
\end{align}
where $\sigma(x) = $ max$\{x,0\}$ is the rectified linear unit (ReLU) activation function and $\boldsymbol{\theta} = [$vec$(\mathbf{W}_0)^T,...,$vec$(\mathbf{W}_{L})^T]^T \in \mathbb{R}^p$ is the weight vector of the neural network where $p=md+m^2(L-1)+m$ where $m$ is the width of the hidden layers and for simplicity of analysis we assume the width is the same for every layer. We assume a depth of $L+1$ to ensure that we have $L$ hidden layers. This will simplify the notation in our later theoretical analysis. %we have in the proofs. 
To denote the gradient of the neural network we use $\mathbf{g}(\mathbf{x};\boldsymbol{\theta}) = \nabla_{\boldsymbol{\theta}}f(\mathbf{x};\boldsymbol{\theta}) \in \mathbb{R}^p$. We make the following assumption about the base arm reward function $h$, which makes sense in many recommendation applications since the base arm reward function takes as input the average preference vector of a user according to the genres and the genre vector of a movie to output a rating. Hence if a user has a higher preference towards the genres of the movie, then the rating should also increase accordingly.

\begin{assumption} \label{continuity_h}
(Lipschitz Continuity) ${h}(\mathbf{x}_{t,i})$ is Lipschitz continuous with respect to the context vector $\mathbf{x}_{t,i}$ of arm $i$, i.e. $\exists B' > 0$ such that for any $\mathbf{x}_{t,i}, \mathbf{x}_{t,i'}$ it holds that $|{h}(\mathbf{x}_{t,i})-{h}(\mathbf{x}_{t,i'}) |\leq B' \left \lVert \mathbf{x}_{t,i}-\mathbf{x}_{t,i'} \right \rVert_1$.
\end{assumption}

\subsection{Super Arm Reward Function}
The super arm reward function $R(S,\mathbf{r})$ is a function of the selected arms' rewards and gives an understanding about how good the selected arm combination is. Let us denote the base arm rewards vector of a super arm $S$ in round $t$ with context vector $\mathbf{x}_{t,S} = [\mathbf{x}_{t,i}]_{i \in S}$ by $\mathbf{r_t}$, %\carlee{$i \in S$}, 
and we denote the expected super arm reward by $\mathbb{E}[R(S_t,\mathbf{r}_t)] = u(S_t,\mathbf{h}(\mathbf{x}_{t,S_t}))$ where $\mathbf{h}(\mathbf{x}_{t,S_t})=[h(\mathbf{x}_{t,i})]_{i \in S}$. In this work the super arm reward function can be any function which satisfies the following mild assumptions, which are common in the combinatorial bandit setting \cite{pmlr-v48-lif16,Qin2014ContextualCB,pmlr-v108-nika20a}.
%\begin{restatable} {assumption}
\begin{assumption} \label{continuity_u}
(Lipschitz Continuity) $u(S,\mathbf{h})$ is Lipschitz continuous with respect to the expected base arm rewards $\mathbf{h}$, % vector which includes the arms in $S$, 
i.e. $\exists \ B > 0$ such that for any $\mathbf{h}, \mathbf{h'}$ it holds that $|u(S,\mathbf{h})-u(S,\mathbf{h'})|\leq B {\sum_{i\in S}|h_i-h'_i|}$.
\end{assumption}
%\end{restatable}
\begin{assumption} \label{monotonicity_u}
(Monotonicity)  $u(S,\mathbf{h})$ is monotone non-decreasing with respect to the expected base arm reward vector $\mathbf{h}$ so that for any possible $S$ when $h_i \leq h'_i,  \forall i\in [N]$ then $u(S,\mathbf{h})\leq u(S,\mathbf{h'})$. 
\end{assumption}

% \begin{assumption} \label{continuity_h}
% (Lipschitz Continuity) ${h}(\mathbf{x}_{t,i})$ is Lipschitz continuous with respect to the context vector $\mathbf{x}_{t,i}$ vector which is the context vector of arm $i$, i.e. $\exists B' > 0$ such that for any $\mathbf{x}_{t,i}, \mathbf{x}_{t,i'}$ it holds that $|{h}(\mathbf{x}_{t,i})-{h}(\mathbf{x}_{t,i'}) |\leq B' \left \lVert \mathbf{x}_{t,i}-\mathbf{x}_{t,i'} \right \rVert_1$. \carlee{move this to the base arm reward subsection}
% \end{assumption}

It should be noted that \textit{these assumptions do not restrict the generalizability of our work}, since they hold in many real life applications of combinatorial bandits. \rev{For example in the previously mentioned movie recommendation setup, the super arm reward function could be the average rating given by the users to the movie that was recommended, which is then monotonic and Lipschitz continuous in each individual rating.} Moreover, \textit{we do not require an oracle that knows how to select arms which would maximize the super arm reward function}: as explained in Algorithm \ref{Algortihm 1}, we learn this through clustering and using the monotonicity assumption. %\carlee{Also say here that we do not require an oracle that knows how to maximize this reward function: we learn this through the clustering.}

To learn the super arm reward function $R(S,\mathbf{r})$, we use a structurally constrained neural network, inspired by the approach in MonoNet \cite{nguyen2023mononet} which is composed of monotonically connected layers, ensuring a monotonic relationship between the inputs and the output. Let us denote layer $k$ of the network by $\mathbf{o}^{(k)}$ for $k=0,1,...,L_m$ and hence we can write 
%\begin{align}
$\mathbf{o}^{(k+1)} = \sigma(\boldsymbol{\Theta}^{(k)}\mathbf{o}^{(k)}+\mathbf{b}^{(k)})$,
%\end{align}
where $\boldsymbol{\Theta}^{(k)}$ is the weight matrix and $\mathbf{b}^{(k)}$ is the bias. The idea to form a monotonic layer is to transform the weight matrix such that it only has non-negative entries, since this will ensure a monotonic relationship between the input and output of that layer. Hence for a monotonic layer we have: 
\begin{align}
\mathbf{o}^{(k+1)} = \sigma(q(\boldsymbol{\Theta}^{(k)})\mathbf{o}^{(k)}+\mathbf{b}^{(k)})   
\end{align}
 where $q$ is the weight transform function applied element wise to every entry of the weight matrix and could be
 % $q(\Theta) = $ $ 
 %  \Theta^{2}$, $q(\Theta)=e^{\Theta}/(1+e^{\Theta})
 % $ etc., or 
 any function with nonnegative outputs. We use a neural network with same width for every layer $n$ and with depth $L_m+1$. By structuring the super arm network so as to guarantee monotonicity, we ensure that the monotonicity assumption holds and \textit{allow our theoretical guarantees to hold} as well. 

\subsection{Regret Formulation}
The main objective of the agent in this setting is to minimize the \textbf{cumulative expected super arm regret}, which is the standard performance metric for CC-MAB problems and measures the gap in reward between the optimal and selected arms. 
%Also it is important to note that the notion of a sublinear regret in the neural bandit setting is quite different from the other bandit settings, in which the goal is to show that the regret-per-round goes to zero as $T$ increases, i.e., that the regret is sublinear in $T$. Since in neural bandits, if we use a very deep network to do the learning, $p$ can be quite large. Hence it becomes more important that the regret is sublinear in $p$ rather than $T$ however the bound that we derive is sublinear in both $p$ and $T$. 
The regret is formally defined as follows: 
%\begin{small}
\begin{align}\label{regr_form}
\mathcal{R}(T) &=  %\sum_{t=1}^{T} ( \max_{S_t \in \mathcal{S}}u(S_t,\mathbf{h}(\mathbf{x}_{t,S_t})) -u(S_t,\mathbf{h}(\mathbf{x}_{t,S_t}))) \nonumber\\&
 \sum_{t=1}^{T} ( u(S_t^*,\mathbf{h}(\mathbf{x}_{t,S_t^*})))  -u(S_t,\mathbf{h}(\mathbf{x}_{t,S_t}))) 
\end{align}
%\end{small}
%\carlee{I think $u$ was defined as $R$ previously. Also need to define $\mathbf{h}$. It may be worth noting that these are the actual reward functions, not the learned ones.}\com{Actually $u$ is the expected super arm reward function whereas $R$ is the super arm reward function, we also defined $\mathbf{h}$ in section 2.4(Total Reward Function)}

\section{Proposed NeUClust Algorithm} \label{sec:algo}
%\carlee{We don't need a subsection here}
In this section, we present our algorithm Neural UCB Clustering (NeUClust). %\tdCarlee{Neural UCB Clustering, so that the U stands for UCB?} 
 NeUClust is a neural network based contextual combinatorial upper confidence bound bandit algorithm that makes use of two neural networks to learn the base arm reward function and the super arm reward function. %Since the super arm selection step makes use of the optimism in the face of uncertainty (OIFU) principle \cite{lai1985asymptotically} by constructing optimistic guesses of the rewards of the arms, this is a UCB type algorithm for a combinatorial setting.  %In addition, we utilize a monotonic neural network composed of monotonic layers as in \cite{nguyen2023mononet} which guarantees a monotonic relationship betweeen the input (base arm rewards) and the output (super arm reward). Clustering of the context space either in an online or offline fashion is also necessary to incorporate the super arm network in the UCB estimate and to guide the super arm selection. 
However, since we do not utilize an oracle, \textit{choosing a super arm based on the optimism in the face of uncertainty (OIFU) principle \cite{lai1985asymptotically} with the super arm network is difficult}. We overcome this challenge by exploiting our monotonicity assumption and \textit{clustering} the context space. Intuitively, base arms with similar contexts should yield similar rewards, and thus are similarly likely to be part of the optimal super arm. In some settings, we can cluster the context space offline, e.g., we may have sufficient information about users movie genre preference to cluster users offline, but in other settings we may need to collect context observations while running NeUClust. %Clustering the context space offline means that we have access to sufficient amount of offline data which would help us cluster the contexts just once before we run our algorithm and clustering in an online fashion means that we have access to the context of the arms in each round and we cluster the context at the start of each round. 
We thus propose online and offline variants of our NeUClust algorithm; we present the online version in this section.%\carlee{explain a bit more what online vs. offline means and what is required to use one or the other (e.g., access to offline data)}

%This approach is helpful since we make use of the structure of the super arm reward function to accelerate the training process.

%\carlee{reference line numbers when describing the algorithm}

\noindent \textbf{NeUClust description.} Algorithm \ref{Algortihm 1} shows NeUClust's pseudocode (in Appendix). %\carlee{Unfortunately since we took out the pseudocode we should remove the line numbers. It might be good to put more paragraph breaks and subheadings to separate the algorithm stages (done)} 
We first initialize the parameters of the first neural network, which aims to learn the base arm reward function, by randomly generating $\boldsymbol{\theta}_0 = [$vec$(\mathbf{W}_0)^T,...,$vec$(\mathbf{W}_{L})^T]^T$, where for each $\ell \in [L]$, $\mathbf{W}_\ell = (\mathbf{W},0;0,\mathbf{W})$ where each entry of $\mathbf{W}$ is generated independently from $\mathcal{N}(0,4/m)$ and $\mathbf{W}_{L} = (\mathbf{w}^T,-\mathbf{w}^T)$ with each entry of $\mathbf{w}$ independently sampled from $\mathcal{N}(0,2/m)$. We initialize the parameters of the second neural network, which tries to learn the super arm reward function,  by independently sampling each entry from $\mathcal{N}(1/n,1)$ where $n$ is the width of the network. %For the online variant of this algorithm, 

At the start of each round $t \in [T]$, the agent observes the contexts of all arms $\{\mathbf{x}_{t,i}\}_{i \in [N] }$ and clusters the arms based on their contexts. NeUClust then uses the context information, and the output of the base and super arm neural networks and its gradients to construct an upper confidence bound (UCB) $v_{t,c_j}$ of the expected reward for every arm $c_j$ in every cluster $c$. We then select the top $K$ arms of each cluster based on the arm UCBs. The outputs of the base arm network (i.e., estimates of base arm reward) are fed through a ReLU function, to ensure they are nonnegative, and then taken as input to the second network to estimate the super arm reward. Then the upper confidence bound vector $\mathbf{V}_t$ is used to select the cluster which has the highest average UCB of its corresponding arms. %We pass the base arm network outputs through a ReLU function before passing them to the super arm network which ensures the inputs to the super arm network are non-negative. 
Then the top $K$ arms of the chosen cluster are played as  $S_t$ and the base arm rewards $\{r_{t,i}\}_{i \in S_t}$ and  super arm reward $R_t$ are observed. 

Once the rewards $\{r_{t,i}\}_{i \in S_t}$ and $R_t$ are observed, the context vector is updated according to the observed base arm rewards of the played arms. The neural network parameters $\boldsymbol{\theta}_t$ and $\boldsymbol{\Theta}_t$ are updated by using gradient descent with step size $\eta_1$ and $\eta_2$ for $J$ iterations to minimize the loss functions: %\carlee{define $\mathbf{f}$}
%\begin{small}
\begin{align*}
    &\mathcal{L}(\boldsymbol{\theta}) = \frac{1}{2} \sum_{t'=1}^{t}\left \lVert \mathbf{f}(\mathbf{x}_{t',S_{t'}};\boldsymbol{\theta})-\mathbf{r}_{t'})\right \rVert_2^2 + \frac{m\lambda_1}{2} \left \lVert \boldsymbol{\theta}-\boldsymbol{\theta}_0\right \rVert_2^2\\
    &\mathcal{L}(\boldsymbol{\Theta}) = \sum_{t'=1}^{t}\frac{({F}(\mathbf{f}(\mathbf{x}_{t',S_{t'}});\boldsymbol{\Theta})-{R}_{t'}))^2 + n\lambda_2\left \lVert \boldsymbol{\Theta}-\boldsymbol{\Theta}_0 \right \rVert_2^2}{2} 
\end{align*}
%\end{small}
\noindent where $\mathbf{f}(\mathbf{x}_{t',S_{t'}};\boldsymbol{\theta})=[f(\mathbf{x}_{t,i})]_{i\in S_{t'}}$, $\mathcal{L}(\boldsymbol{\theta})$ shows the loss minimized using $l_2$-regularization for the base arm neural network and $\mathcal{L}(\boldsymbol{\Theta})$ is the loss minimized for the super arm neural network. $F$ represents the super arm network here and $R_{t'}$ is the observed super arm reward of the selected super arm $S_{t'}$ and $\mathbf{r}_{t'}$ represents the vector of observed base arm rewards of the selected super arm. The hyperparameters $\lambda_1$ and $\lambda_2$ adjust for the level of regularization which centers at the randomly initialized weight vectors. We define the positive exploration scaling factor $\gamma_t$ similar to \cite{pmlr-v202-hwang23a,neuralucbzhou20a} as%\tdCarlee{do we need these equations here? It doesn't look like they appear in the rest of the main body, and it's not clear how they are used}
\begin{align}
   \gamma_t = \Gamma_{1,t} \left(\zeta\sqrt{\log\frac{\det \mathbf{Z}_t}{\det \lambda_1 \mathbf{I}} + \Gamma_{2,t} - 2\log \Delta} + \sqrt{\lambda_1 }S\right) + \nonumber\\
    (\lambda_1 + C_1 t K L) \left((1 - \eta_1 m \lambda_1)^{\frac{J}{2}} \sqrt{\frac{tK}{\lambda_1}} + \Gamma_{3,t} \right) 
\label{eq_8}    
\end{align}
for some $\Gamma_{1,t},\Gamma_{2,t},\Gamma_{3,t}>0$ (defined later in Appendix).

% where
% \[
% \begin{small}
% \begin{aligned}
% \Gamma_{1,t} &= \sqrt{1 + C_{\Gamma,1} t^{\frac{7}{6}} K^{\frac{7}{6}} L^{4} \lambda_1^{-\frac{7}{6}} m^{-\frac{1}{6}} \sqrt{\log m}}, \\
% \Gamma_{2,t} &= C_{\Gamma,2} t^{\frac{5}{3}} K^{\frac{5}{3}} L^{4} \lambda_1^{-\frac{1}{6}} m^{-\frac{1}{6}} \sqrt{\log m}, \\
% \Gamma_{3,t} &= C_{\Gamma,3} t^{\frac{7}{6}} K^{\frac{7}{6}} L^{\frac{7}{2}} \lambda_1^{-\frac{7}{6}} m^{-\frac{1}{6}} \sqrt{\log m}(1 + \sqrt{tK/\lambda_1}),
% \end{aligned}
% \end{small}
% \]
% for some constants $C_1, C_{\Gamma,1}, C_{\Gamma,2}, C_{\Gamma,3} > 0$ and $\Delta \in (0,1]$.

\textbf{Super arm selection without an oracle.} As discussed above, our super arm selection first finds the optimal cluster (in terms of the super arm reward) and then selects the $K$ arms with largest base arm rewards from that cluster. By monotonicity of the super arm reward, we thus maximize the super arm reward subject to the constraint that all base arms belong to the chosen cluster. For example, if a user generally likes movies of a particular genre, or more generally prefers a certain kind of product, our clustering should group movies of similar genres (products of similar type) together and then recommend them to the user. While restricting the selected base arms to a specific cluster may lead to a suboptimal reward, experimental results show that NeUClust outperforms neural combinatorial bandit algorithms that do not have this clustering restriction: such algorithms generally impose other restrictions on the super arm reward, e.g., the presence of an oracle. Moreover, we experimentally verify that real-world recommendation datasets naturally have a clustered structure in the Appendix.
%To select the super arm in each round, unlike many combinatorial bandit works in the literature which use an $\alpha$-approximation or exact oracle to select the super arm in round $t$, we do clustering on the contexts of the base arms and choose the optimal cluster and play the "best" arms in that cluster. To select the cluster we make use of the output of the super arm network as well as the base arm network to choose the cluster whose arms have the highest average UCB estimate. Our algorithm selects the best arms in the chosen cluster using the upper confidence bound estimates of the base arms which is composed of the output of the base arm network and its gradients. 

For the \textbf{offline variant} of the algorithm, we do clustering only once at round $t = 1$ and do not update the context vector after observing the base arm rewards of the arms that were played. This algorithm would work well when the contexts have been formed with an abundance of prior offline data. For example for movie recommendation, this would mean that the context vector which represents the average rating the user gave to each genre is formed with a sufficient number of already available ratings from that user. %which would ensure that the context vector would not change much even if we update after observing the reward. %If the users have only rated a couple of movies, then we need to do online clustering, as the context vector would change greatly and would give us a lot of information about the genre preferences of the user.

\begin{algorithm}[b]
\caption{NeUClust (Online)}
\begin{algorithmic}[1]
\STATE\textbf{Input: }{Number of rounds $T$, regularization parameters $\lambda_1$ and $\lambda_2$, step sizes $\eta_1$ and $\eta_2$, number of gradient descent steps $J$, network widths $m$ and $n$, network depths $L$ and $L_m$, number of clusters $M$, maximum number of iterations of clustering $i_c$, size of super arm $K$, norm parameter $S$.}
\STATE Randomly initialize $\boldsymbol{\theta}_0$ and $\boldsymbol{\Theta}_0$ as described previously and $\mathbf{Z}_0 = \lambda_1 \mathbf{I}$ \label{lst:line:initialization}
\FOR{$t = 1, \ldots, T$}
\STATE Observe $\{\mathbf{x}_{t,i}\}_{i \in [N] }$
    \STATE Run $k$-means clustering on the observed contexts $\{\mathbf{x}_{t,i}\}_{i \in [N] }$ with $k=M$ and maximum number of iterations $i_c$
    \FOR{$c = 1, \ldots, M$}
    \FOR{$j \in c $}

    \STATE $v_{t,c_j} = f(\mathbf{x}_{t,c_j};\boldsymbol{\theta}_{t-1}) + \gamma_t || \mathbf{g}(\mathbf{x}_{t,c_j};\boldsymbol{\theta}_{t-1})/ \sqrt{m} || _{\mathbf{Z}_{t-1}^{-1}}$ 
    \ENDFOR
    
    \STATE Within cluster $c$, find top $K$ cluster elements with the highest $v$ value, $\mathbf{b}_c \in \mathbb{R}^K$ vector includes the indices of these elements 
    \STATE Initialize $\mathbf{V}_t = (0,0,...,0) \in \mathbb{R}^{M}$
    \FOR{$a \in \mathbf{b}_c$}
    \STATE $V_{t,c} \pluseq v_{t,c_a} + F(\mathbf{f}(\mathbf{x}_{t,\mathbf{b}_c};\boldsymbol{\theta}_{t-1}))$
    \ENDFOR
    
    \ENDFOR
    \STATE $c' = \argmax_c(\mathbf{V}_{t})$
    \STATE $S_t = \mathbf{b}_{c'} $

\STATE Play super arm $S_t$ and observe base arm rewards $\{r_{t,i}\}_{i \in S_t}$ and super arm reward $R_t$

    \STATE Update context vector $\{\mathbf{x}_{t,i}\}_{i \in S_t }$ of the arms played by using the observed base arm rewards 
    
\STATE Update $\mathbf{Z}_t = \mathbf{Z}_{t-1} + \sum_{i \in S_t}\mathbf{g}(\mathbf{x}_{t,i};\boldsymbol{\theta}_{t-1})\mathbf{g}(\mathbf{x}_{t,i};\boldsymbol{\theta}_{t-1})^T/m$
    \STATE Update weights $\boldsymbol{\theta}_t$ by minimizing the loss $\mathcal{L}(\boldsymbol{\theta})$ using gradient descent with step size $\eta_1$ for $J$ iterations.
    \STATE Update weights $\boldsymbol{\Theta}_t$ by minimizing the loss $\mathcal{L}(\boldsymbol{\Theta})$ using gradient descent with step size $\eta_2$ for $J$ iterations.
    \STATE Update $\gamma_t$ as explained in Eq. \ref{eq_8} %\carlee{cite equation number}
\ENDFOR

\end{algorithmic}
\label{Algortihm 1}
\end{algorithm}

\section{Regret Analysis} \label{sec:regret_analysis}
In this section we analyze and derive the regret of the presented algorithm NeUClust in Algorithm 1. We denote by $\{\mathbf{x}^{i}\}_{i = 1}^{TN}$ the collection of all the contexts $\{\mathbf{x}_{1,1},\ldots,\mathbf{x}_{T,N}\}$. 

%\carlee{using amsthm environments for definitions and assumptions will make it easier to refer back to them later}
\begin{definition}
\cite{ntkjacot,cao2019} For a given set of contexts $\{\mathbf{x}^{i}\}_{i = 1}^{TN}$ define
%\begin{small}
\begin{align*}
&\widetilde{\mathbf{H}}^{(1)}_{i,j} = \boldsymbol{\Sigma}_{i,j}^{(1)}=\langle \mathbf{x}^i,\mathbf{x}^j \rangle,
  \mathbf{A}_{i,j}^{(\ell)} = \begin{pmatrix}
\boldsymbol{\Sigma}_{i,i}^{(\ell)} & \boldsymbol{\Sigma}_{i,j}^{(\ell)} \\
\boldsymbol{\Sigma}_{i,j}^{(\ell)} & \boldsymbol{\Sigma}_{j,j}^{(\ell)} 
\end{pmatrix}\\
&\boldsymbol{\Sigma}_{i,j}^{(\ell+1)} = 2\mathbb{E}_{(u,y)\sim \mathcal{N}(\boldsymbol{0},\mathbf{A}_{i,j}^{(\ell)})}[\sigma(u)\sigma(y)]\\
&\widetilde{\mathbf{H}}^{(\ell+1)}_{i,j} = 2\widetilde{\mathbf{H}}^{(\ell)}_{i,j}\mathbb{E}_{(u,y)\sim \mathcal{N}(\boldsymbol{0},\mathbf{A}_{i,j}^{(\ell)})}[\sigma'(u)\sigma'(y)]+\boldsymbol{\Sigma}_{i,j}^{(\ell+1)}.
\end{align*}
%\end{small}
Then $\mathbf{H} = (\widetilde{\mathbf{H}}^{(L)}+\boldsymbol{\Sigma}^{(L)})/2$ is called the \textbf{neural tangent kernel} (NTK) matrix on the context set $\{\mathbf{x}^{i}\}_{i = 1}^{TN}$. 
\end{definition}

\noindent The NTK matrix $\mathbf{H}$ is defined iteratively from the input to the output layer of the $L$ layer neural network \cite{neuralucbzhou20a,zhang2021neural,pmlr-v202-hwang23a}. 
\begin{assumption}
For any $i \in  [TN]$, $\|\mathbf{x}^i\|_2 = 1$ and $[\mathbf{x}^i]_j = [\mathbf{x}^i]_{j+d/2}$ for $1 \leq j \leq d/2$ and for some $\lambda_0 \geq 0$, $\mathbf{H} \succeq \lambda_0 \mathbf{I}$.   
\end{assumption}

\noindent This is a fairly relaxed assumption which is commonly used in the neural bandit literature \cite{pmlr-v202-hwang23a,zhang2021neural,neuralucbzhou20a}. The condition on the norm is for ease of analysis and we can always form a new context $\mathbf{x'} = [\mathbf{x}^T,\mathbf{x}^T]^T/\sqrt{2}$ to satisfy the assumption on the entries of $\mathbf{x}$. An assumption is also made on the positive semi-definiteness of $\mathbf{H}$, which is also standard. 
\begin{definition}
The \textbf{effective dimension} $\widetilde{d}$ of the NTK matrix $\mathbf{H}$ with regularization parameter $\lambda_1$ is defined as 
\begin{align*}
    \widetilde{d} = \frac{\log \det(\mathbf{I}+\mathbf{H}/\lambda_1)}{\log(1+TN/\lambda_1)}
\end{align*}    
\end{definition}

\noindent The effective dimension roughly quantifies the number of dimensions of contexts in the Reproducable Kernel Hilbert Space which the NTK matrix spans. It was first introduced by \cite{valko} for kernelised contextual bandits and we adapt the definition proposed by \cite{yang2019reinforcement}.
\begin{assumption} \label{clustering_assump}
 (Clustering) Let $\mathcal{A}_t$ be the set of clusters that we find in round $t$ and $\mathcal{C}$ the ground truth clusters of the context space. Without loss of generality, suppose $\mathcal{A}_t$ and $\mathcal{C}$ are ordered so that mapping each cluster $j\in\mathcal{A}_t$ to cluster $j\in\mathcal{C}$ minimizes the average distance between cluster centers in $\mathcal{A}_t$ and $\mathcal{C}$. Then the dot product of the mean contexts of clusters $c$ and $c'$ satisfies $\bar{\mathbf{x}}_{c,t} \cdot \mathbf{x}_{c'} \leq \delta_t $ $\forall t\in[T]$ and $\forall c' \neq c  $ for some $\delta_t > 0$ which is a decreasing function. %where $\bar{\mathbf{x}}_{c,t}$ denotes the mean context vector of our chosen cluster $c$ at round $t$ according to our clustering scheme and $\mathbf{x}_{c'}$ is the mean context vector of an actual cluster $c'$. 
 Further, $\bar{\mathbf{x}}_{c,t} \cdot \mathbf{\bar{x}}_{c_t^*} \geq P_t $ for some $P_t >0$ which is a non-decreasing function of $t$, where $c_t^*$ denotes the optimal cluster for a given round, $\mathbf{\bar{x}}_{c_t^*}$ is its mean context vector, and we play the super arm in cluster $c$ at round $t$.%\carlee{Do we need $\delta_t\rightarrow 0$? Otherwise this is more like a definition. We also need to relate $c'$ and $c$ more clearly, I think, since cluster indices are order-agnostic (done)}
\end{assumption}
\begin{assumption} \label{ell0_assump}
    (Nonzero mean contexts)
    $\forall t \in [T], \forall c'\in \mathcal{C} $, $\left \lVert  \mathbf{\bar{x}}_{c,t} \right \rVert_0 =0,  \left \lVert \mathbf{x}_{c'} \right \rVert_0 =0 $, i.e., $\mathbf{\bar{x}}_{c,t}$ and $\mathbf{x}_{c'}$ have no 0 entries. $\mathcal{C}$ is the set of ground truth clusters and we play the super arm in cluster $c$ at round $t$. %\tdCarlee{what is $c'$ here?}
\end{assumption}

\noindent Assumption \ref{clustering_assump} essentially states that we are able to identify the optimal cluster and separate the clusters more accurately over time, as we select $\delta_t$ to be a decreasing and $P_t$ to be a non-decreasing function; we also implicitly assume that the clusters are sufficiently different from one another, which can be assured by shifting and normalizing the contexts. %In every round $t$, we have a set of clusters $\mathcal{A}_t$ which results from our clustering in that round. There also exists a set of actual clusters $\mathcal{C}$ which does not change over time. In every round, for the cluster $c$ that we choose to select the super arm, we identify the closest cluster (using inter-cluster distance) in $\mathcal{C}$ and subtract it from the set $\mathcal{C}$ and the $\delta_t$ part applies $\forall c' \in \mathcal{C}$ except for the closest one. The second part of the assumption applies for the optimal cluster of that round and our chosen cluster. 
Assumption \ref{ell0_assump} states that the mean cluster vector of our chosen cluster in any round and the mean cluster vectors of the actual clusters have no zero elements. This is necessary for our proofs but is not very restrictive in practice since the condition is on the mean vectors. Next we give a condition on the width of the base arm network, which is necessary since our theoretical results hold for overparameterized networks.
\begin{condition}
The width of the base arm network satisfies
%\begin{small}
\begin{align*}
    &m \geq C\max\big\{T^6 N^6 L^6 \lambda_0^{-1} \log(T^2 N^2 L/\Delta)\max(\lambda_0^{-4},1),
    \\
    &L^{-\frac{3}{2}}K^{-\frac{1}{2}}\lambda_1^{\frac{1}{2}}(\log(TNL^2/\Delta))^{\frac{3}{2}}\big\}
    \\
    &m(\log m)^{-3} \geq \max\big\{C\lambda_1^{-1}L^{24}K^{10}T^{10},
    \\    
    &CT^7K^7L^{14}\lambda_1^{-7}(\lambda_1+TKL)^6(1+\sqrt{TK/\lambda_1})^6\big\}.
\end{align*}
%\end{small}
\end{condition}
\noindent for some constant $C>0.$ While some of these conditions on $m$ are similar to those on the width of the network in \cite{pmlr-v202-hwang23a,neuralucbzhou20a}, we include extra terms due to the neural network used to estimate the super arm reward.
%\carlee{explain how this is related to the condition in the other Neural UCB paper}
%\carlee{Name the lemmas and explain which are new (and why they are nontrivial).}
%\theorem

While much of our regret bound proof follows the outline of~\cite{pmlr-v202-hwang23a}'s proof of a regret bound for neural contextual combinatorial bandits, \cite{pmlr-v202-hwang23a} crucially does not use a neural network to estimate the super arm, instead employing a known oracle. We next outline the \textbf{new theoretical results} that we need to extend the proof to our setting, which chiefly involves bounding the output of the super arm neural network and incorporating the error from clustering.
\begin{lem} \label{lemma 4.1}
(Bound on base arm network output) With probability $(1 - O(L) \cdot e^{-\Omega(m\varepsilon^2/L)})(1 - e^{-\Omega(m\omega^\frac{2}{3} L)})(1-me^{(-m\rho^2/4)})$ the output of the first neural network
%\begin{small}
\begin{align*}
    f(\mathbf{x}^i) \leq 4\sqrt{1/\lambda_1}(9T^{-\beta\alpha+0.5})
\end{align*}
%\end{small}
for some $\beta > 0$, $\alpha > 0$ and $\omega \leq O(L^{-\frac{9}{2}}(\log m)^{-3})$.
\end{lem}

Here $\varepsilon$  represents the largest error %in the regression over $T$ rounds, i.e., the error 
between the predicted output of the first neural network and the actual base arm reward of the arm selected over $T$ rounds as in \cite{pmlr-v97-allen-zhu19a}. The above lemma bounds the output of the base arm neural network for all time steps, which we use to bound the output of the super arm network as shown in the next lemma. To prove Lemma~\ref{lemma 4.1}, we make use of Lemmas 7 and 8, which requires a perturbation to the weights and the $\ell_2$ norm of that perturbation is upper bounded by $\omega$. We show that this perturbation is within the values the actual weights take when we do Gradient Descent. The \textit{main challenge} is showing this bound holds for all time steps and that the weights of the network stay within the desired region.
\begin{lem} \label{lemma 4.2}
(Bound on super arm network output and expected base arm reward vector)
With probability $(1 - O(L) \cdot e^{-\Omega(m\varepsilon^2/L)})(1 - e^{-\Omega(m\omega^\frac{2}{3} L)})(1-me^{(-m\rho^2/4)})(1-\Delta)$ the output of the super arm network and the base arm reward vector for the chosen super arm $S_t$ are bounded by
%\begin{small}
\begin{align*}
    &F(\mathbf{f}(\mathbf{x}_{t,S_t})) \leq 4\sqrt{1/\lambda_1}(9T^{-\beta\alpha+0.5}) K n^{L_m} \\
    &\mathbf{h}(\mathbf{x}_{t,S_t}) \leq \mathbf{v}_{t,S_t} +  \mathbf{1}^K \cdot  F(\mathbf{f}(\mathbf{x}_{t,S_t})).
\end{align*}
%\end{small}
% \[ 
% &F(\mathbf{f}(\mathbf{x}_{t,S_t})) \leq 4\sqrt{m/\lambda_1}(9T^{-\beta\alpha+0.5}) K n^{L_m} \\
% &\mathbf{h}(\mathbf{x}_{t,S_t}) \leq \mathbf{v}_{t,S_t} +  F(\mathbf{f}(\mathbf{x}_{t,S_t}))
% \]
\end{lem}
\noindent This lemma makes use of the fact that the super arm network takes as input the base arm network outputs and this bound is also used to upper bound the expected base arm reward. This lemma helps bound the expected super arm regret of a single round, which will be done in the following lemma.
\begin{lem} \label{lemma 4.3}
(Bound on expected regret for a single round)
For any \(\Delta \in (0,1)\), suppose the width of the neural network \(m\) satisfies Condition 1. If \(\eta_1   \leq C_1(T K m L + m\lambda_1)^{-1}\), and \(\lambda_1 \geq C_2 L K\), for some positive constant \(C_1, C_2\) with \(C_2 \geq \sqrt{\max_{t,i} \left\|\mathbf{g}(\mathbf{x}_{t,i};\boldsymbol{\theta}_{t-1})/\sqrt{m}\right\|_2^2/L}\) and let \(\gamma_t\) be as defined in Equation \ref{eq_8}. Then the expected super arm regret of a single round is bounded with probability at least $1-\Delta$ as 
%\begin{small}
\begin{align*}
    & u(\mathbf{h}(\mathbf{x}_{t,S_t^*}))-u(\mathbf{h}(\mathbf{x}_{t,S_t})) \leq \underbrace{D\sqrt{{\sum_{i\in S_t}{\|\mathbf{g}(\mathbf{x}_{t,i};\boldsymbol{\theta}_{t-1})/\sqrt{m}\|_{\mathbf{Z}_{t-1}^{-1}}^2}}}}_{\text{$Y_{t}$}}
    \\
    &+\underbrace{BKe_t}_{\text{$U_t$}} +  \underbrace{BK^2 4\sqrt{1/\lambda_1}(9T^{-\beta\alpha+0.5}) n^{L_m}}_{\text{$P$}}+\\
    &\underbrace{B'BK\left( \frac{C^{''}Q\delta_t \sqrt{d}N}{P_t Y} +  \frac{NQ\delta_t\sqrt{d}}{Y} \right)}_{\text{$E_t$}}
\end{align*}
%\end{small}
where $e_t := C_3\gamma_{t-1} t^{\frac{1}{6}} K^{\frac{1}{6}} L^{\frac{7}{2}} m^{-\frac{1}{6}}\lambda_1^{-\frac{2}{3}} \sqrt{\log m}
+ C_4 t^{\frac{2}{3}} K^{\frac{2}{3}} m^{-\frac{1}{6}}\lambda_1^{-\frac{2}{3}} \sqrt{\log m},$ $ D = 2B \gamma_{T}\sqrt{K}$ for some absolute constants $C_3,C_4>0$. 
\end{lem}
\noindent 
As can be seen from the decomposition of the per round regret in Lemma \ref{lemma 4.3}, the regret bound is separated into different terms where there is a term coming from clustering $(E_{t})$, a term coming from the super arm network $(P)$, and two terms coming from the base arm network $(Y_{t}, U_{t})$. This decomposition in fact carries over to the expected super arm regret over the time horizon $T$, which we use Lemma \ref{lemma 4.3} to bound in the following theorem.

\begin{thm} \label{thm_1}
(Regret Bound)
Supposing that Assumptions \ref{continuity_h}-\ref{ell0_assump} hold, let $\mathbf{h}=[h(\mathbf{x}^i)]_{i=1}^{TN}]$, and select the hyperparameters of the algorithm as follows %\carlee{add $\delta_t$ assumption} \com{Done}
%\begin{small}
\begin{align*}
    &\eta_1 = C_1(T K m L + m\lambda_1)^{-1}
    \\    &\frac{16m\cdot81T^{(-2\beta\alpha+1)}K^2n^{2L_m}}{e_T^2} \geq\lambda_1\geq C_2LK
    \\
    &J = 2\log\left(\sqrt{\lambda_1/(TK)}/(\lambda_1+C_3TKL)\right)TKL/(C_1\lambda_1) 
    \\
    &S \geq \sqrt{2\mathbf{h}^T\mathbf{H}^{-1}\mathbf{h}}
\end{align*}
%\end{small}
for constants $C_1,C_2,C_3$ defined as in Lemma 3. Then if $m$ is chosen such that it satisfies Condition 1, the cumulative expected regret of NeUClust over $T$ rounds is bounded by 
%\begin{small}
\begin{align*}
    \mathcal{R}(T) = \widetilde{O}\left(\widetilde{d}\sqrt{T}\right)
\end{align*}
%\end{small}
\end{thm} %\carlee{We say in the regret definition that it is important for regret to be sublinear in $p$; do we show that in this expression? (removed that definition for now)}

%\carlee{add two discussion paragraphs: one on the separation of the bound into different terms and one about dependence on $K$ and how that relates to prior works' dependence on $K$}
%(I think the  two paragraphs below give good intuition but I was wondering whether we should keep it or not, since keeping it would mean that we also need to include the whole regret bound which takes a lot of space. Another possibility is labelling the terms in Lemma 3 and talking about them here.)

%\textcolor{red}{I think this paragraph and the one below give good intuition but I was wondering whether we should keep it or not, since keeping it would mean that we also need to include the whole regret bound which takes a lot of space.}
\noindent We thus see that our total regret grows as $\sqrt{T}$. which matches the dependence in prior combinatorial neural bandit work \cite{pmlr-v202-hwang23a}.

\textbf{Dependence on $K$ and $\delta_t$.} From our regret bound it seems that we are super linear in the super arm size $K$ due to the $P$ term but we can select $\lambda_1$ such that we get $\sqrt{K}$ dependence. Similarly the clustering term $\delta_t$ could be chosen such that we get sublinear regret in $K$ for $E_{t}$. The only unavoidable linear $K$ term is the one coming from $U_{t}$. Comparing with the bound of \cite{pmlr-v202-hwang23a} where they get a sublinear bound in terms of $K$, it should be noted that this is due to the use of the oracle in their algorithm, and the main reason we get linear dependence on $K$ for our bound is because we do not use such an oracle.

From this bound, we see that if $\delta_t$ is a constant or non-decreasing with $t$ or $P_t$ is decreasing with $t$, we could incur linear regret. This is intuitive with the formation of our NeUClust algorithm: if the context space does not have an underlying clustered structure or we cannot cluster users with similar preferences correctly, then since super arm selections are guided by clusters we would not make optimal or near optimal arm selections, incurring linear regret.

%\com{Add discussion on clustering term and its implications}

\section{Experiments} \label{sec:experiments}
In this section, we measure and evaluate the performance of NeUClust on two recommendation datasets: MovieLens 25M Dataset \footnote{\url{https://grouplens.org/datasets/movielens/}} and the Yelp Open Dataset\footnote{\url{https://www.yelp.com/dataset}}. We compare our algorithm with the following state of the art combinatorial and neural bandit algorithms: (1) \textbf{CN-UCB} \cite{pmlr-v202-hwang23a} which uses a single neural network to approximate the base arm rewards of the arms and uses an exact oracle to select super arms, (2) \textbf{Neural-MAB} \cite{shouxu_c2ucb} which uses two networks to learn the base arm and super arm reward functions and selects the super arm by going over arms individually and estimating their contribution to the super arm reward, (3) \textbf{CC-MAB} \cite{chen2018ccmab} which partitions the context space into hypercubes for exploration and exploitation, and (4) \textbf{$K$-LinUCB} \cite{linucb} which is a combinatorial version of the LinUCB algorithm and assumes linear realizability.
\subsection{Experiment Setup} \label{sec:movie_recom}
Due to a lack of space, we omit the description of the Yelp restaurant recommendation setup here and include it in the Appendix. It generally follows that of the MovieLens movie recommendation setup described below.

We assume that every round we need to recommend an incoming movie to $K$ users, hence each base arm corresponds to a user. %and the super arm is a collection of $K$ users which are selected for the movie to be recommended to. 
The context of a user is represented as its user preference vector, which is a vector showing the average rating the user gave to movies they rated for the $d = 20$ genres. For example, if a user has only watched and rated crime, thriller and horror films, the context of the user will be composed of zeros for the remaining 17 genres. %and for instance for the context element corresponding to crime genre, it will be an average rating value for the movies the user watched that had crime as one of its genres. 
We use this context information to cluster users and identify user types, i.e. users which prefer similar genres. The base arm reward is a binary indicator (0 or 1) of whether the user liked the movie that was recommended (i.e., gave a rating $\geq 4.0$). The super arm reward is a binary indicator of whether at least \%80 of the $K$ selected users liked the movie (i.e., $\geq$ \%80 of users gave a rating $\geq 4.0$); this reward is highly nonlinear and captures the overall success of our recommendations. As a movie comes in round $t$, we use the following function to generate the true mean/quality of base arm (user) $i$ for that movie $z_t$
\begin{small}
\begin{align*}
    \mu_i = 2/\left(1+\exp\left(-\frac{\langle\mathbf{a}_{z_t},\mathbf{x}_{t,i}\rangle}{2\|\mathbf{a}_{z_t}\odot\mathbf{x}_{t,i}\|_0}\right)\right)-1
\end{align*}
\end{small}
\noindent where here $\langle \cdot \rangle $ and $\odot$ represents the inner product and elementwise product respectively, $\mathbf{a}_{z_t}$ represents the genre vector of the movie and $\mathbf{x}_{t,i}$ represents the context of the base arm. We define the mean of a base arm in this way since as the genre vector and context of a user are more similar, it is more likely that the user will like the movie. We let the base arm reward of arm $i$ be defined by the function $r(\mathbf{x}_{t,i}) \sim Bern(\mu_i)$ to capture uncertainty in a user's ratings for movies of a given genre. % Such a definition seems reasonable due to fact that there is some randomness and uncertainty in whether a user gives a high rating to a movie even if the genres of the movie align with the genres a user generally prefers. 
The super arm reward of a super arm $S_t$ is defined as follows $R(S_t,\mathbf{r}_t) = 1$ if $\sum_{k=1}^Kr_{t,k} \geq 0.8K$ else $R(S_t,\mathbf{r}_t) = 0$. %The main goal in movie recommendation is to ensure the movie(s) you recommend to the users are highly rated by most of the users which indicates the success of the recommendation. This is our main motivation for modeling the super arm reward function in this way. 
%(\underbrace{|v_{t,i}-h(\mathbf{x}_{t,i})|}_{\text{$I_1$}}

\subsection{Experimental Results}
We run our experiments for the offline variant of NeUClust (Algorithm 2, pseudocode in Appendix). Due to a lack of space, we give the explanation of the hyperparameter tuning and experimental details for all algorithms in the Appendix. The given results for the MovieLens dataset are the mean over 10 independent runs for each of the algorithms, whereas for the Yelp dataset the mean is taken over 5 independent runs. The error bars in the figures indicate $\pm$ std. In the Appendix, we give further experimental results on NeUClust's effectiveness and the validity of our clustering assumption.
% , showing how the number of clusters affects our algorithms performance, the within cluster sum of squares plots for the two datasets and the super arm reward figures for the two datasets.

Figures~\ref{fig:movielens} and~\ref{fig:yelp} show the regret achieved by all algorithms on the MovieLens and Yelp datasets, respectively. We can see from both figures that \textit{our proposed NeUClust algorithm outperforms the other baseline algorithms} in terms of super arm regret for both datasets. The regret gap  even grows over time. The results indicate that unlike some algorithms, NeUClust's regret does not suffer from the increase in the context dimension in going from MovieLens (20-dimensional context) to Yelp (171-dimensional context); indeed, NeUClust exhibits a larger reward gap with the baselines for Yelp compared to MovieLens. As expected, CN-UCB performs most similarly to our proposed algorithm due to its similar algorithmic structure to NeUClust, even though the two algorithms have major differences related to their super arm selection and NeUClust's usage of a second neural network for super arm reward estimation.

We have excluded the results of the Yelp experiments for the CCMAB algorithm since the performance was considerably worse both in terms of runtime and regret, which might have been due to the larger context space that could have impacted the discretization approach used by CCMAB.% \carlee{can shorten this paragraph (done)} 

To compare the performance of the baselines and our proposed algorithm, we include the following table which displays the final cumulative regret and the runtime in seconds of the algorithms. These results illustrate that despite the use of a second neural network, NeUClust performs comparably in terms of runtime while achieving lower regret.
\begin{table}[h]
    \centering
    \resizebox{\columnwidth}{!}
    {
    \begin{tabular}{|c|c|c|c|c|c|}
        \hline
        & Regret (ML) & Runtime (ML) & Regret (Y) & Runtime (Y) \\ 
        \hline
        NeUClust & \textbf{152} & 486 &\textbf{157} & 1976\\ 
        \hline
         CN-UCB & 168 & 332 & 182&1562  \\ 
        \hline
        CCMAB&  403&\textbf{141} &-  &-    \\
        \hline
        K-LinUCB & 207 & 1352  & 282&81293 \\ 
        \hline
         NeuralMAB & 227 & 187 &185 &\textbf{1317} \\
         \hline
    \end{tabular}
    }
    \caption{%\carlee{Invert this table so that the algorithms are the rows and Regret and Runtime for MovieLens and Yelp are the columns} 
    Cumulative regret and runtime (rounded in seconds) comparison of the algorithms for the MovieLens (ML) and Yelp (Y) datasets. NeUClust's runtime is significantly below K-LinUCB's but slightly higher than CN-UCB and NeuralMAB due to using a second neural network to estimate super-arm rewards.}
    \label{tab:runtime}
\end{table}
\begin{figure}[t]
\centering
\includegraphics[width=\linewidth,height=5cm,keepaspectratio]{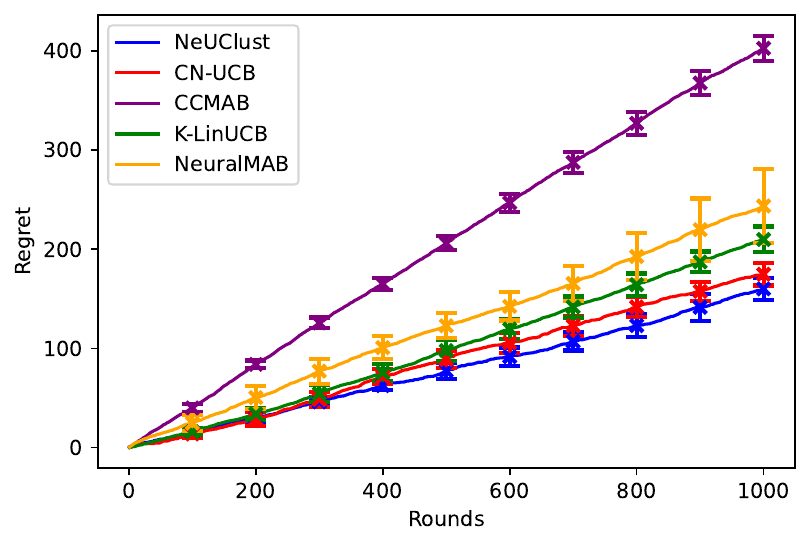}
\caption{Super arm regret plot for the MovieLens dataset. Our NeUClust algorithm has lower regret than all baselines.} %\carlee{can make the plot shorter}}
\label{fig:movielens}
\end{figure}

\begin{figure}[t]
\centering
\includegraphics[width=\linewidth,height=5cm,keepaspectratio]{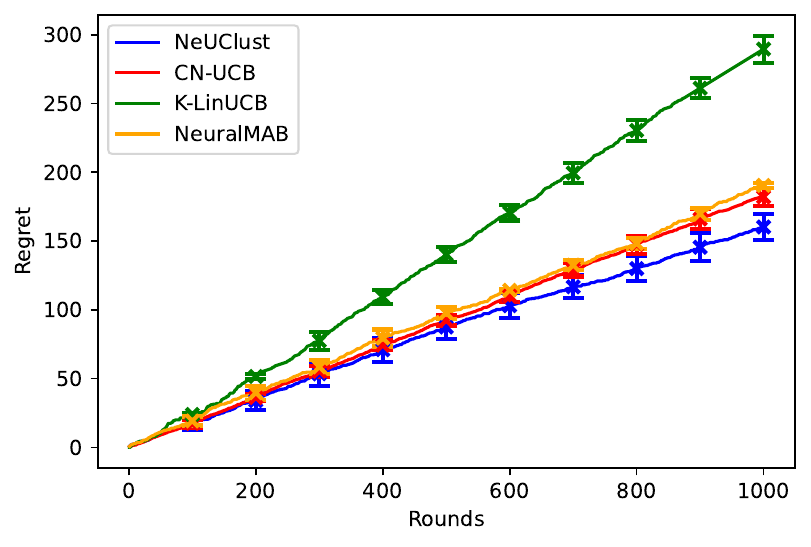}
\caption{Super arm regret plot for the Yelp dataset. Our NeUClust algorithm has lower regret than all baselines.}
\label{fig:yelp}
\end{figure}
\section{Conclusion}\label{sec:conclusion}
In this paper, we focus on the contextual combinatorial bandit problem and its application to recommender systems. We propose a novel solution algorithm NeUClust, which uses two neural networks to estimate the reward functions and exploits the underlying structure in the context space by doing clustering to guide super arm selections. Using theoretical results of previous works for overparameterized networks and proposing new lemmas to handle clustering and neural network contributions, we prove that our algorithm achieves $\widetilde{O}\left(\widetilde{d}\sqrt{T}\right)$ cumulative regret. We also run experiments on recommendation datasets which illustrate that our algorithm competes well with other algorithms. Future work includes further investigation of the use of different clustering methods within NeUClust, especially in other applications.
%\pagebreak
\bibliography{aaai25}

\section{Reproducibility Checklist}
This paper
\begin{itemize}
    \item Includes a conceptual outline and/or pseudocode description of AI methods introduced: Yes
    \item Clearly delineates statements that are opinions, hypothesis, and speculation from objective facts and results: Yes
    \item Provides well marked pedagogical references for less-familiare readers to gain background necessary to replicate the paper: Yes
\end{itemize}
Does this paper make theoretical contributions? Yes
\begin{itemize}
    \item All assumptions and restrictions are stated clearly and formally: Yes
    \item All novel claims are stated formally (e.g., in theorem statements): Yes
    \item Proofs of all novel claims are included: Yes
    \item Proof sketches or intuitions are given for complex and/or novel results: Yes
    \item Appropriate citations to theoretical tools used are given: Yes
    \item All theoretical claims are demonstrated empirically to hold: Yes
    \item All experimental code used to eliminate or disprove claims is included: Yes
\end{itemize}
Does this paper rely on one or more datasets? Yes
\begin{itemize}
    \item A motivation is given for why the experiments are conducted on the selected datasets: Yes
    \item All novel datasets introduced in this paper are included in a data appendix: Yes
    \item All novel datasets introduced in this paper will be made publicly available upon publication of the paper with a license that allows free usage for research purposes: Yes
    \item All datasets drawn from the existing literature (potentially including authors’ own previously published work) are accompanied by appropriate citations: Yes
    \item All datasets drawn from the existing literature (potentially including authors’ own previously published work) are publicly available: Yes
    \item All datasets that are not publicly available are described in detail, with explanation why publicly available alternatives are not scientifically satisficing: Yes
\end{itemize}
Does this paper include computational experiments? Yes
\begin{itemize}
    \item Any code required for pre-processing data is included in the appendix: Yes
    \item All source code required for conducting and analyzing the experiments is included in a code appendix: Yes
    \item All source code required for conducting and analyzing the experiments will be made publicly available upon publication of the paper with a license that allows free usage for research purposes: Yes
    \item All source code implementing new methods have comments detailing the implementation, with references to the paper where each step comes from: Yes
    \item If an algorithm depends on randomness, then the method used for setting seeds is described in a way sufficient to allow replication of results: Yes
    \item This paper specifies the computing infrastructure used for running experiments (hardware and software), including GPU/CPU models; amount of memory; operating system; names and versions of relevant software libraries and frameworks: Yes
    \item This paper formally describes evaluation metrics used and explains the motivation for choosing these metrics: Yes
    \item This paper states the number of algorithm runs used to compute each reported result: Yes
    \item Analysis of experiments goes beyond single-dimensional summaries of performance (e.g., average; median) to include measures of variation, confidence, or other distributional information: Yes
    \item The significance of any improvement or decrease in performance is judged using appropriate statistical tests (e.g., Wilcoxon signed-rank): No since this is not applicable in this line of research.
    \item This paper lists all final (hyper-)parameters used for each model/algorithm in the paper’s experiments: Yes
    \item This paper states the number and range of values tried per (hyper-) parameter during development of the paper, along with the criterion used for selecting the final parameter setting: Yes
\end{itemize}
\appendix
\pagebreak
\section{Proofs} \label{sec:proofs}

%\subsection{Lemma 1} 
\textbf{Lemma 1.} With probability $(1 - O(L) \cdot e^{-\Omega(m\varepsilon^2/L)})(1 - e^{-\Omega(m\omega^\frac{2}{3} L)})(1-me^{(-m\rho^2/4)})$ the output of the first neural network 
\begin{align*}
    f(\mathbf{x}^i) \leq 4\sqrt{1/\lambda_1}(9T^{-\beta\alpha+0.5})
\end{align*}
 for some $\beta > 0$, $\alpha > 0$ and $\omega \leq O(L^{-\frac{9}{2}}(\log m)^{-3})$.
\begin{proof}
The last layer of the first neural network has weights $\mathbf{W}_{L}^{(0)}$ initialized and independently sampled from $\mathcal{N}(0,2/m)$. We know for the last layer that $\mathbf{W}_{L}^{(0)} \in \mathbb{R}^m$ and $\mathbf{W}_{L}^{(0)} \sim \mathcal{N}(0,2/m)$. By applying the Gaussian tail bound we get for any $\rho \geq 0$ and $\forall k \in [m]$
\begin{align*}
    &P(W_{L,k}^{(0)} \geq \rho) \leq e^{-\rho^2/2(2/m)}=e^{-m\rho^2/4}\\
    &1-e^{-m\rho^2/4} \leq 1- P(W_{L,k}^{(0)} \geq \rho) = P(W_{L,k}^{(0)} \leq \rho)\\
    &P(\mathbf{W}_L^{(0)} \leq \rho) \geq 1-me^{-m\rho^2/4}
\end{align*}
where the first and second inequality follow by the Gaussian tail bound, and the third inequality follows from union bound. It should be noted that since we have an over parameterized network $m$ is relatively large and hence this event happens with very high probability. 

By Lemma 7, for an $\varepsilon \in (0,1]$ which represents the largest error in the regression over $T$ rounds, i.e. the error between the predicted output of the first neural network and the actual base arm reward of the arm selected, for the last layer at initialization we have that 
\begin{align*}
    \left \lVert \mathbf{s}^{(0)}_{L}\right \rVert_2 \in [1 - \varepsilon, 1 + \varepsilon]
\end{align*}
By Lemma 8 we are required to have $\overrightarrow{\mathbf{W}}'$ which corresponds to $\mathbf{W}_L$ in our case. Hence we have $ \|\mathbf{W}_L-\mathbf{W}_L^{(0)}\|_2 \leq \|\mathbf{W}_L\|_2 + \|\mathbf{W}_L^{(0)}\|_2 \leq \omega + \sqrt{m}\rho$ which happens with probability at least $1-me^{-m\rho^2/4}$. Given that by Lemma 9 we know that with probability at least $1-\Delta$, $\|\mathbf{W}_L^{(t)}-\mathbf{W}_L^{(0)}\|_2 \leq 2\sqrt{t/(m\lambda_1)} \leq 2\sqrt{T/(m\lambda_1)}$ we can select $\rho$ such that $\omega + \sqrt{m}\rho \geq 2\sqrt{T/(m\lambda_1)}$ which would mean that we can choose $\rho = 2\frac{\sqrt{T/\lambda_1}}{m}$. Also the above arguments imply that $\forall k \in [m]$, $|W_{L,k}^{(t)}| \leq 2\sqrt{T/\lambda_1}+m\rho \leq 4\sqrt{T/\lambda_1}$ by $\ell_1$-$\ell_2$ norm inequality. This guarantees that 
\begin{align*}
    \|\mathbf{s}_{L}- \mathbf{s}_{L}^{(0)}\|_2 \leq O(\omega L^{\frac{5}{2}}\sqrt{\log m})
\end{align*}
 and since $\omega \leq O(L^{-\frac{9}{2}}\log^{-3} m)$ and 
\begin{align*}
    \|\mathbf{s}_{L}- \mathbf{s}_{L}^{(0)}\|_2 \geq \|\mathbf{s}_{L}\|_2-\|\mathbf{s}_{L}^{(0)}\|_2
\end{align*}
  it follows that $\|\mathbf{s}_{L}\|_2 \leq 2+\varepsilon$. By previous arguments we have that $\|\mathbf{W}_L^{(t)}-\mathbf{W}_L^{(0)}\|_2 \leq  2\sqrt{T/(m\lambda_1)}$ which implies $\|\mathbf{W}_L^{(t)}\|_2-\|\mathbf{W}_L^{(0)}\|_2 \leq 2\sqrt{T/(m\lambda_1)}$ and $\|\mathbf{W}_L^{(t)}\|_2 \leq 2\sqrt{T/(m\lambda_1)}+\sqrt{m}\rho$. Now by the $\ell_1$-$\ell_2$ norm inequality we have that 
\begin{gather*}
    \|\mathbf{W}^{(t)}_{L}\|_2 \leq \|\mathbf{W}^{(t)}_{L}\|_1 \leq \sqrt{m}\|\mathbf{W}^{(t)}_{L}\|_2 \leq 2\sqrt{T/\lambda_1}+m\rho\\
    \sum_{z=1}^{m}W^{(t)}_{L,z} \leq \|\mathbf{W}^{(t)}_{L}\|_1 \leq 2\sqrt{T/\lambda_1}+m\rho \leq 4\sqrt{T/\lambda_1} 
\end{gather*}
 where $z \in [m]$ and $W^{(t)}_{L,z}$ denotes the $z$-th weight of the hidden layer $L$ at iteration $t$. Given that this event happens with probability $(1 - O(L) \cdot e^{-\Omega(m\varepsilon^2/L)})(1 - e^{-\Omega(m\omega^\frac{2}{3} L)})$ due to Lemma 7 and the fact that the activation function is ReLU, the following holds with probability at least $(1 - O(L) \cdot e^{-\Omega(m\varepsilon^2/L)})(1 - e^{-\Omega(m\omega^\frac{2}{3} L)})(1-me^{(-m\rho^2/4)})(1-\Delta):$
\begin{align*}
    &f(\mathbf{x}^i) \leq (2+\varepsilon) \sum_{z=1}^{m}W^{(t)}_{L,z} \\
    &f(\mathbf{x}^i) \leq (2+\varepsilon)4\sqrt{T/\lambda_1}\\
    &f(\mathbf{x}^i) \leq 4(2+\varepsilon)^2 \sqrt{T/\lambda_1}\\ 
    &f(\mathbf{x}^i) \leq 4(9\varepsilon^{\alpha}) \sqrt{T/\lambda_1}   
\end{align*}
where $\alpha = \log_{\varepsilon}\frac{4+2\varepsilon}{9}$ for some $\varepsilon$.  Letting $\varepsilon = T^{-\beta}$ for some $\beta > 0$ 
\begin{align*}
    f(\mathbf{x}^i) \leq 4\sqrt{1/\lambda_1}(9T^{-\beta\alpha+0.5})
\end{align*}
\end{proof}
\textbf{Lemma 2.} With probability $(1 - O(L) \cdot e^{-\Omega(m\varepsilon^2/L)})(1 - e^{-\Omega(m\omega^\frac{2}{3} L)})(1-me^{(-m\rho^2/4)})(1-\Delta)$ the output of the super arm network and the base arm reward vector for the chosen super arm $S_t$ are bounded by
\begin{align*}
    &F(\mathbf{f}(\mathbf{x}_{t,S_t})) \leq 4\sqrt{1/\lambda_1}(9T^{-\beta\alpha+0.5}) K n^{L_m} \\
    &\mathbf{h}(\mathbf{x}_{t,S_t}) \leq \mathbf{v}_{t,S_t} + \mathbf{1}^K\cdot   F(\mathbf{f}(\mathbf{x}_{t,S_t}))
\end{align*}
\begin{proof}
Given that by Lemma 1 we know that the output of the first neural network is bounded for all time steps by $f(\mathbf{x}^i) \leq 4\sqrt{1/\lambda_1}(9T^{-\beta\alpha+0.5})$ we can now bound the output of the second neural network by exploiting its structure. As the weight transformation function $\mathbf{q}$ we use a sigmoid function which means that the weights of the network are bounded $(0,1)$. We denote by $\mathbf{o}_\ell$ the output of the $\ell$-th hidden layer of the second network where $\ell \in [L_m]$. $o_{\ell,z}$ denotes the $z$-th output of the $\ell$-th hidden layer where $z \in [K]$. Hence it follows that $\forall z \in [K]$
\begin{align*}
    &{o}_{1,z} \leq 4\sqrt{1/\lambda_1}(9T^{-\beta\alpha+0.5}) K \\
    &{o}_{2,z} \leq 4\sqrt{1/\lambda_1}(9T^{-\beta\alpha+0.5}) K n \\
    &{o}_{L_m,z} \leq 4\sqrt{1/\lambda_1}(9T^{-\beta\alpha+0.5}) K n^{L_m-1} \\
    &F(\mathbf{f}(\mathbf{x}_{t,S_t})) \leq 4\sqrt{1/\lambda_1}(9T^{-\beta\alpha+0.5}) K n^{L_m} 
\end{align*}
where the first inequality follows due to the fact that the first hidden layer has $K$ inputs. The second, third and fourth inequalities follow similarly due to the other hidden layers having input dimension $n$ which is also the width of the network.

By Corollary 1 we know that the base arm reward function $h$ is upper bounded by the sum of their UCB term $u_{t,i}$ and offset term $e_t$ hence to prove the lemma we need 
\begin{align*}
    \mathbf{u}_{t,S_t} +\mathbf{1}^K\cdot e_t \leq \mathbf{v}_{t,S_t} + \mathbf{1}^K \cdot F(\mathbf{f}(\mathbf{x}_{t,S_t})) 
\end{align*}
where $\mathbf{u}_{t,S_t}$ is a vector of the UCB terms of the selected super arm $S_t$ and similarly $\mathbf{v}_{t,S_t}$ is the vector of our UCB term for the selected super arm. Since $\mathbf{Z}_t$ and $\gamma_t$ are the same in our algorithm and that proposed by \cite{pmlr-v202-hwang23a}, for the same $S_t$ (implying same contexts) we have $\mathbf{u}_{t,S_t} = \mathbf{v}_{t,S_t}$. Hence we only need to show
\begin{align*}
    e_t \leq e_T \leq F(\mathbf{f}(\mathbf{x}_{t,S_t})) \leq 4\sqrt{1/\lambda_1}(9T^{-\beta\alpha+0.5}) K n^{L_m}
\end{align*}
To show that $F(\mathbf{f}(\mathbf{x}_{t,S_t})) \geq e_T$ we let $\kappa > 0$ represent the lower bound on all the weights and $\epsilon>0$ is a lower bound on the sum of the base arm network outputs (the lower bound $\epsilon>0$ since as mentioned in the Proposed NeUClust Algorithm section we pass the base arm network outputs through a ReLU function before feeding them as input to the super arm network. This implies that $F(\mathbf{f}(\mathbf{x}_{t,S_t})) \geq \epsilon \kappa^{L_m+1}n^{L_m} \geq e_T$ for large $m$ which means that 
\begin{align*}
    \mathbf{h}(\mathbf{x}_{t,S_t}) \leq \mathbf{v}_{t,S_t} + \mathbf{1}^K\cdot F(\mathbf{f}(\mathbf{x}_{t,S_t}))
\end{align*}
\end{proof}
\textbf{Lemma 3.}
For any \(\Delta \in (0,1)\), suppose the width of the neural network \(m\) satisfies Condition 1. If \(\eta_1   \leq C_1(T K m L + m\lambda_1)^{-1}\), and \(\lambda_1 \geq C_2 L K\), for some positive constant \(C_1, C_2\) with \(C_2 \geq \sqrt{\max_{t,i} \left\|\mathbf{g}(\mathbf{x}_{t,i};\boldsymbol{\theta}_{t-1})/\sqrt{m}\right\|_2^2/L}\) and let \(\gamma_t\) be a positive scaling factor defined as
\begin{align*}
   \gamma_t = \Gamma_{1,t} \left(\zeta\sqrt{\log\frac{\det \mathbf{Z}_t}{\det \lambda \mathbf{I}} + \Gamma_{2,t} - 2\log \Delta} + \sqrt{\lambda_1}S\right) + \\
    (\lambda_1 + C_1 t K L) \left((1 - \eta_1 m \lambda_1)^{\frac{J}{2}} \sqrt{\frac{tK}{\lambda_1}} + \Gamma_{3,t} \right), 
\end{align*}
where
\[
\begin{aligned}
\Gamma_{1,t} &= \sqrt{1 + C_{\Gamma,1} t^{\frac{7}{6}} K^{\frac{7}{6}} L^{4} \lambda_1^{-\frac{7}{6}} m^{-\frac{1}{6}} \sqrt{\log m}}, \\
\Gamma_{2,t} &= C_{\Gamma,2} t^{\frac{5}{3}} K^{\frac{5}{3}} L^{4} \lambda_1^{-\frac{1}{6}} m^{-\frac{1}{6}} \sqrt{\log m}, \\
\Gamma_{3,t} &= C_{\Gamma,3} t^{\frac{7}{6}} K^{\frac{7}{6}} L^{\frac{7}{2}} \lambda_1^{-\frac{7}{6}} m^{-\frac{1}{6}} \sqrt{\log m}(1 + \sqrt{tK/\lambda_1}),
\end{aligned}
\]
for some constants $C_1, C_{\Gamma,1}, C_{\Gamma,2}, C_{\Gamma,3} > 0$. Then the expected super arm regret of a single round is bounded with probability at least $1-\Delta$ as 
\begin{align*}
    & u(\mathbf{h}(\mathbf{x}_{t,S_t^*}))-u(\mathbf{h}(\mathbf{x}_{t,S_t})) \leq D\sqrt{{\sum_{i\in S_t}{\|\mathbf{g}(\mathbf{x}_{t,i};\boldsymbol{\theta}_{t-1})/\sqrt{m}\|_{\mathbf{Z}_{t-1}^{-1}}^2}}}
    \\
    &+BKe_t +  BK^2 4\sqrt{1/\lambda_1}(9T^{-\beta\alpha+0.5}) n^{L_m}+\\
    &B'BK\left( \frac{C^{''}Q\delta_t \sqrt{d}N}{P_t Y} +  \frac{NQ\delta_t\sqrt{d}}{Y} \right)
\end{align*}
where $e_t := C_3\gamma_{t-1} t^{\frac{1}{6}} K^{\frac{1}{6}} L^{\frac{7}{2}} m^{-\frac{1}{6}}\lambda_1^{-\frac{2}{3}} \sqrt{\log m}
+ C_4 t^{\frac{2}{3}} K^{\frac{2}{3}} m^{-\frac{1}{6}}\lambda_1^{-\frac{2}{3}} \sqrt{\log m},$ $ D = 2B \gamma_{T}\sqrt{K}$ for some absolute constants $C_3,C_4>0$. 

\begin{proof}
As seen in equation (5), the regret equation shows that the expected regret of a single round is 
\begin{align*}
    u(S_t^*,\mathbf{h}(\mathbf{x}_{t,S_t^*}))  -u(S_t,\mathbf{h}(\mathbf{x}_{t,S_t}))
\end{align*}
where we will instead use the notation $u(\mathbf{h}(\mathbf{x}_{t,S_t}))$ for simplicity. By our clustering  assumption (Assumption \ref{clustering_assump}) we have 
\begin{align*}
    & \bar{\mathbf{x}}_{c,t}\cdot\mathbf{x}_{c'}  \leq \delta_t \\
    & \left \lVert\bar{\mathbf{x}}_{c,t}\right \rVert_2 \left \lVert\mathbf{x}_{c'}\right \rVert_2  \leq Q\bar{\mathbf{x}}_{c,t}\cdot\mathbf{x}_{c'}  \leq Q\delta_t\\
    &\left \lVert\bar{\mathbf{x}}_{c,t}\right \rVert_2 \leq \ \frac{Q\delta_t}{\left \lVert\mathbf{x}_{c'}\right \rVert_2}\\
    &\left \lVert \mathbf{x}_{c'}\right \rVert_2 \geq Y \\
    &\left \lVert\bar{\mathbf{x}}_{c,t}\right \rVert_2 \leq \ \frac{Q\delta_t}{Y}\\
    &\left \lVert\bar{\mathbf{x}}_{c,t}\right \rVert_1 \leq \sqrt{d}\left \lVert\bar{\mathbf{x}}_{c,t}\right \rVert_2 \leq \frac{Q\delta_t\sqrt{d}}{Y}
    % &|u(\mathbf{h}(\mathbf{x}_{t,S_t^*})) - u(\mathbf{h}(\mathbf{x}_{t,S_t}))| \leq B \left \lVert \mathbf{h}(\mathbf{x}_{t,S_t^*})-\mathbf{h}(\mathbf{x}_{t,S_t}) \right \rVert_1\\
    % &u(\mathbf{h}(\mathbf{x}_{t,S_t^*})) - u(\mathbf{h}(\mathbf{x}_{t,S_t})) \leq |u(\mathbf{h}(\mathbf{x}_{t,S_t^*})) - u(\mathbf{h}(\mathbf{x}_{t,S_t}))| \\
    % & u(\mathbf{h}(\mathbf{x}_{t,S_t^*})) - u(\mathbf{h}(\mathbf{x}_{t,S_t})) \leq B \left \lVert \mathbf{h}(\mathbf{x}_{t,S_t^*})-\mathbf{h}(\mathbf{x}_{t,S_t}) \right \rVert_1  
\end{align*}
where the second inequality follows by the Kantorovich inequality where $Q= \frac{A}{G}>0$ and $A=(V+Z)/2, G=\sqrt{VZ}$ and $0< V\leq\frac{\bar{x}_{i,t}}{x_{j'}}\leq Z<\infty$ which holds $\forall i\in c$ and $\forall j'\in c'$. It should be noted that for the Kantorovich inequality to hold we require $\bar{\mathbf{x}}_{c,t}$ and $\mathbf{x}_{c'}$ to have non-zero elements (Assumption \ref{ell0_assump}) which is not a very restrictive constraint since the non-zero constraint is on the average context vector of the empirical cluster and the actual cluster. A violation of this constraint would only occur if none of the users belonging to a cluster have rated a specific movie genre which would result in the average rating of the context corresponding to that movie genre to be zero. The third inequality follows trivially. The fourth inequality lower bounds the $\ell_2$ norm of all the actual clusters $c'$ and $Y>0$ here is an absolute constant. The fifth inequality follows by previous arguments and the last inequality follows by the $\ell_1-\ell_2$ norm inequality. By definition $\bar{\mathbf{x}}_{c,t}$ is the average cluster vector of arms assigned to cluster $c$ in round $t$, hence
\begin{align*}
    & \bar{\mathbf{x}}_{c,t} = \frac{\sum_{i\in c}x_{i,t}}{|c|}
    % &\left \lVert \mathbf{h}(\mathbf{x}_{t,S_t^*})-\mathbf{h}(\mathbf{x}_{t,S_t}) \right \rVert_1 \leq B' \left \lVert \mathbf{x}_{t,S_t^*}-\mathbf{x}_{t,S_t} \right \rVert_1
    % \\
    % &u(\mathbf{h}(\mathbf{x}_{t,S_t^*})) - u(\mathbf{h}(\mathbf{x}_{t,S_t})) \leq B'B\left \lVert \mathbf{x}_{t,S_t^*}-\mathbf{x}_{t,S_t} \right \rVert_1
    % \\
    % &u(\mathbf{h}(\mathbf{x}_{t,S_t^*}))-u(\mathbf{h}(\mathbf{x}_{t,S_t})) \leq u(\mathbf{v}_{t,S_t} +  F(\mathbf{f}(\mathbf{x}_{t,S_t}))+\delta_t\cdot \mathbf{1})-u(\mathbf{h}(\mathbf{x}_{t,S_t}))
    % \\
    % &u(\mathbf{h}(\mathbf{x}_{t,S_t^*}))-u(\mathbf{h}(\mathbf{x}_{t,S_t})) \leq B\sum_{i\in S_t}|v_{t,i}+F(\mathbf{f}(\mathbf{x}_{t,S_t})+\delta_t-h(\mathbf{x}_{t,i})|
    % \\
    % &u(\mathbf{h}(\mathbf{x}_{t,S_t^*}))-u(\mathbf{h}(\mathbf{x}_{t,S_t})) \leq B\sum_{i\in S_t}(\underbrace{|v_{t,i}-h(\mathbf{x}_{t,i})|}_{\text{$I_{1,t}$}}+\underbrace{|F(\mathbf{f}(\mathbf{x}_{t,S_t})|}_{\text{$I_{2,t}$}}+\delta_t)
\end{align*}
which implies that since all context elements are non-negative and the chosen super arm at round $t$ is a subset of the cluster that was chosen
\begin{align*}
    &\left \lVert {\mathbf{x}}_{t,S_t} \right \rVert_1 \leq |c_t| \left \lVert \bar{\mathbf{x}}_{c_t,t} \right \rVert_1\\
    &\left \lVert {\mathbf{x}}_{t,S_t} \right \rVert_1 \leq |c_t| \frac{Q\delta_t\sqrt{d}}{Y}\\
\end{align*}
where $c_t$ represents the cluster chosen at round $t$. The second inequality follows from previous arguments. It should be noted that $\mathbf{\bar{x}}_{c_t}$ is the average context vector of the arms in cluster $c_t$. The arms in the cluster is fixed throughout the time horizon as we assume the existence of a clustered structure between the arms which does not change over time. However, the context of the arms can change over time, as for example in the case of movie recommendation the context of an arm becomes average ratings for different genres which is possible to change over time. 
By Assumption \ref{continuity_h} we have that 
\begin{align*}
    &|{h}(\mathbf{x}_{t,i^*})-{h}(\mathbf{x}_{t,i}) |\leq B' \left \lVert \mathbf{x}_{t,i^*}-\mathbf{x}_{t,i} \right \rVert_1\\
    & {h}(\mathbf{x}_{t,i^*})-{h}(\mathbf{x}_{t,i}) \leq B' \left \lVert \mathbf{x}_{t,i^*}-\mathbf{x}_{t,i} \right \rVert_1\\
    & {h}(\mathbf{x}_{t,i^*})  \leq B' \left \lVert \mathbf{x}_{t,i^*}-\mathbf{x}_{t,i} \right \rVert_1 +{h}(\mathbf{x}_{t,i})\\
    & \mathbf{h}(\mathbf{x}_{t,S_t^*})  \leq  \mathbf{1}^K \cdot B'\left \lVert \mathbf{x}_{t,S_t^*}-\mathbf{x}_{t,S_t} \right \rVert_1 +\mathbf{h}(\mathbf{x}_{t,S_t})\\
\end{align*}
where the first three inequalities follow through Assumption \ref{continuity_h} and the last inequality follows since the third inequality holds $\forall i\in S_t \And \forall i^* \in S_t^*$. Now bounding $B' \left \lVert \mathbf{x}_{t,S_t^*}-\mathbf{x}_{t,S_t} \right \rVert_1$ we have
\begin{align*}
   & B' \left \lVert \mathbf{x}_{t,S_t^*}-\mathbf{x}_{t,S_t} \right \rVert_1 \leq B'\left( \left \lVert \mathbf{x}_{t,S_t^*}\right \rVert_1 + \left \lVert \mathbf{x}_{t,S_t}\right \rVert_1 \right)\\
   %&B' \left \lVert \mathbf{x}_{t,{S}^*_t}-\mathbf{x}_{t,S_t} \right \rVert_1 \leq B'\left(N(t)Vd+|c_t| \frac{Q\delta_t\sqrt{d}}{Y}\right)
   &P_t \leq \mathbf{\bar{x}}_{c,t} \cdot \mathbf{\bar{x}}_{c_t^*} \leq \left \lVert \mathbf{\bar{x}}_{c,t} \right \rVert_2 \left \lVert \mathbf{\bar{x}}_{c_t^*} \right \rVert_2 \\
   &P_t \left \lVert \mathbf{\bar{x}}_{c_t^*} \right \rVert_2 \leq
   \left \lVert \mathbf{\bar{x}}_{c,t} \right \rVert_2
   \left \lVert \mathbf{\bar{x}}_{c_t^*} \right \rVert_2^2 \leq  \left \lVert \mathbf{\bar{x}}_{c,t} \right \rVert_2 C^{''}\\
   & \left \lVert \mathbf{\bar{x}}_{c_t^*} \right \rVert_2 \leq \frac{C^{''}Q\delta_t}{P_t Y}\\
   & \left \lVert \mathbf{\bar{x}}_{c_t^*} \right \rVert_1 \leq \frac{C^{''}Q\delta_t \sqrt{d}}{P_t Y} \\
   & \left \lVert \mathbf{{x}}_{c_t^*} \right \rVert_1 \leq \frac{C^{''}Q\delta_t \sqrt{d}|c_t^*|}{P_t Y} \\
   & \left \lVert \mathbf{{x}}_{t,S_t^*} \right \rVert_1 \leq \left \lVert \mathbf{{x}}_{c_t^*} \right \rVert_1\\
   & B' \left \lVert \mathbf{x}_{t,S_t^*}-\mathbf{x}_{t,S_t} \right \rVert_1 \leq B'\left( \frac{C^{''}Q\delta_t \sqrt{d}|c_t^*|}{P_t Y} + |c_t| \frac{Q\delta_t\sqrt{d}}{Y} \right)\\
   & B' \left \lVert \mathbf{x}_{t,S_t^*}-\mathbf{x}_{t,S_t} \right \rVert_1 \leq B'\left( \frac{C^{''}Q\delta_t \sqrt{d}N}{P_t Y} +  \frac{NQ\delta_t\sqrt{d}}{Y} \right)  
\end{align*}
where the first inequality follows due to triangle inequality, the second line follows by Assumption \ref{clustering_assump} and the Cauchy-Schwarz inequality, the third line follows by multiplying both sides by $\left \lVert \mathbf{\bar{x}}_{c_t^*}\right \rVert_2$ and by letting $\left \lVert \mathbf{\bar{x}}_{c_t^*}\right \rVert_2^2 \leq C^{''}$ for some $C^{''}>0$ which holds for all $c_t^*$ and all $t$. The next inequality follows by the previously derived bound on $\left \lVert \mathbf{\bar{x}}_{c,t}\right \rVert_2$. The next line comes by the $\ell_1-\ell_2$ norm inequality. We let $\mathbf{x}_{c_t^*}=|c_t^*|\mathbf{\bar{x}}_{c_t^*}$ hence the next inequality follows trivially. For the next inequality, we assume that as the cluster sizes are larger than the super arm size $K$ and as we assume that the optimal super arm comes from the optimal cluster for a given round then it follows that $\left \lVert \mathbf{{x}}_{t,S_t^*} \right \rVert_1 \leq \left \lVert \mathbf{{x}}_{c_t^*} \right \rVert_1$ . The final two lines are a result of combining previous findings and bounding the cardinality of any cluster by the arm size $N$. Given these results we have 
\begin{align*}
    &\mathbf{h}(\mathbf{x}_{t,S_t^*})  \leq  B'\left( \frac{C^{''}Q\delta_t \sqrt{d}N}{P_t Y} +  \frac{NQ\delta_t\sqrt{d}}{Y} \right) +\mathbf{h}(\mathbf{x}_{t,S_t})\\
    &\mathbf{h}(\mathbf{x}_{t,S_t^*})\leq  B'\left( \frac{C^{''}Q\delta_t \sqrt{d}N}{P_t Y} +  \frac{NQ\delta_t\sqrt{d}}{Y} \right)+\mathbf{v}_{t,S_t} + \\   &\mathbf{1}^K \cdot F(\mathbf{f}(\mathbf{x}_{t,S_t}))
\end{align*}
\\\\
where the first inequality follows by previous results and the second inequality follows through Lemma \ref{lemma 4.2}. 
\begin{align*}
    &u(\mathbf{h}(\mathbf{x}_{t,S_t^*}))-u(\mathbf{h}(\mathbf{x}_{t,S_t})) \leq B\left \lVert \mathbf{h}(\mathbf{x}_{t,S_t^*})-\mathbf{h}(\mathbf{x}_{t,S_t})\right \rVert_1\\
    &\leq B \bigg\|  B'\left( \frac{C^{''}Q\delta_t \sqrt{d}N}{P_t Y} +  \frac{NQ\delta_t\sqrt{d}}{Y} \right)+\mathbf{v}_{t,S_t} + \mathbf{1}^K \cdot F(\mathbf{f}(\mathbf{x}_{t,S_t}))\\
    &- \mathbf{h}(\mathbf{x}_{t,S_t})\bigg\|_1\\
    &= B\sum_{i\in S_t}\bigg|B'\left( \frac{C^{''}Q\delta_t \sqrt{d}N}{P_t Y} +  \frac{NQ\delta_t\sqrt{d}}{Y} \right)+{v}_{t,i} +  F(\mathbf{f}(\mathbf{x}_{t,S_t}))\\
    &- {h}(\mathbf{x}_{t,i})\bigg|\\
    &\leq B\sum_{i\in S_t}\Bigg(\underbrace{B'\left( \frac{C^{''}Q\delta_t \sqrt{d}N}{P_t Y} +  \frac{NQ\delta_t\sqrt{d}}{Y} \right)}_{\text{$I_{1,t}$}}+\underbrace{|v_{t,i}-h(\mathbf{x}_{t,i})|}_{\text{$I_{2,t}$}} \\
    &+  \underbrace{|F(\mathbf{f}(\mathbf{x}_{t,S_t}))|}_{\text{$I_{3,t}$}}\Bigg) \\
\end{align*}
where the first inequality follows from the Lipschitz continuity assumption (Assumption \ref{continuity_u}) of the expected super arm reward function, the second inequality follows from previous results, the equality follows by definition and the fourth inequality follows from the triangle inequality. 
\\
We let $\delta_t$ be a decreasing function in $t$ which implies that our clustering scheme is able to identify the underlying clusters such that the average context vector between the empirical cluster is growing further apart from the other clusters which do not correspond to it. $P_t$ is non-decreasing function which governs how fast the average empirical context vector of clusters approach the actual clusters average context vector. However, since we only care about the order of the regret bound, the second term in $I_{1,t}$ determines the order coming from this term since it is not divided by the non-decreasing $P_t$ term.
\begin{align*}
    &I_{1,t} = B'\left( \frac{C^{''}Q\delta_t \sqrt{d}N}{P_t Y} +  \frac{NQ\delta_t\sqrt{d}}{Y} \right)\\
\end{align*}
\\
Now we focus on bounding $I_{2,t}$ by expanding $v_{t,i}$ as in Algorithm 1
\begin{align*}
     &I_{2,t} = \bigg|f(\mathbf{x}_{t,i})+\gamma_{t-1}\sqrt{\mathbf{g}(\mathbf{x}_{t,i};\boldsymbol{\theta}_{t-1})^T\mathbf{\mathbf{Z}}_{t-1}^{-1}\mathbf{g}(\mathbf{x}_{t,i};\boldsymbol{\theta}_{t-1})}\\
     &-h(\mathbf{x}_{t,i})\bigg|   
     % &I_1 \leq |f(\mathbf{x}_{t,S_{t,k}})+\gamma_{t-1}\sqrt{\mathbf{g}(\mathbf{x}_{t,k};\boldsymbol{\theta}_{t-1})^T\mathbf{\widetilde{Z}}_{t-1}^{-1}\mathbf{g}(\mathbf{x}_{t,k};\boldsymbol{\theta}_{t-1})}|+|h(\mathbf{x}_{t,S_{t,k}})|
     % \\
     % &I_1 \leq |f(\mathbf{x}_{t,S_{t,k}})+\gamma_{t-1}\sqrt{\mathbf{g}(\mathbf{x}_{t,k};\boldsymbol{\theta}_{t-1})^T\mathbf{\widetilde{Z}}_{t-1}^{-1}\mathbf{g}(\mathbf{x}_{t,k};\boldsymbol{\theta}_{t-1})}|+|h(\mathbf{x}_{t,S_{t,k}})|
\end{align*}
By Lemma 10 we know that 
\begin{align*}
    I_{2,t} \leq 2 \gamma_{t-1}\|\mathbf{g}(\mathbf{x}_{t,i};\boldsymbol{\theta}_{t-1})/\sqrt{m}\|_{\mathbf{Z}_{t-1}^{-1}} + e_t
\end{align*}
hence we have
\begin{align*}
    &B\sum_{i\in S_t}I_{2,t} \leq B\sum_{i\in S_t}(2 \gamma_{t-1}\|\mathbf{g}(\mathbf{x}_{t,i};\boldsymbol{\theta}_{t-1})/\sqrt{m}\|_{\mathbf{Z}_{t-1}^{-1}} + e_t) 
    % \\ 
    % \leq B\sum_{i\in S_t}\sqrt{(2 \gamma_{t-1}\|\mathbf{g}(\mathbf{x}_{t,i};\boldsymbol{\theta}_{t-1})/\sqrt{m}\|_{\mathbf{Z}_{t-1}^{-1}} + e_t)^2} 
    % \\
    % \leq BK\frac{1}{K}\sum_{i\in S_t}\sqrt{(2 \gamma_{t-1}\|\mathbf{g}(\mathbf{x}_{t,i};\boldsymbol{\theta}_{t-1})/\sqrt{m}\|_{\mathbf{Z}_{t-1}^{-1}} + 2e_t)^2}
    % \\
    % \leq BK \sqrt{\frac{\sum_{i\in S_t}(2 \gamma_{t-1}\|\mathbf{g}(\mathbf{x}_{t,i};\boldsymbol{\theta}_{t-1})/\sqrt{m}\|_{\mathbf{Z}_{t-1}^{-1}} + 2e_t)^2}{K}}
    % \\
    % \leq 4B\sqrt{K}\sqrt{\sum_{i\in S_t}( \max(\gamma_{t-1}\|\mathbf{g}(\mathbf{x}_{t,i};\boldsymbol{\theta}_{t-1})/\sqrt{m}\|_{\mathbf{Z}_{t-1}^{-1}}, e_t))^2}
    \\
    &\ = 2B \gamma_{t-1}\sum_{i\in S_t}(\|\mathbf{g}(\mathbf{x}_{t,i};\boldsymbol{\theta}_{t-1})/\sqrt{m}\|_{\mathbf{Z}_{t-1}^{-1}}) + BKe_t
    \\
    &=  2B \gamma_{t-1}\sum_{i\in S_t}\sqrt{(\|\mathbf{g}(\mathbf{x}_{t,i};\boldsymbol{\theta}_{t-1})/\sqrt{m}\|_{\mathbf{Z}_{t-1}^{-1}})^2} + BKe_t
    \\
    &=  2B \gamma_{t-1}K\frac{1}{K}\sum_{i\in S_t}\sqrt{(\|\mathbf{g}(\mathbf{x}_{t,i};\boldsymbol{\theta}_{t-1})/\sqrt{m}\|_{\mathbf{Z}_{t-1}^{-1}})^2} + BKe_t
    \\
    &\leq 2B \gamma_{T}K\sqrt{\frac{\sum_{i\in S_t}{(\|\mathbf{g}(\mathbf{x}_{t,i};\boldsymbol{\theta}_{t-1})/\sqrt{m}\|_{\mathbf{Z}_{t-1}^{-1}})^2}}{K}} + BKe_t
\end{align*}
where the first inequality follows through Lemma 10, the equalities follow trivially and the final inequality follows through Jensen's inequality since the square root function is a concave function and because $\gamma_{t-1}\leq\gamma_T$. Now focusing on the $I_{3,t}$ term we get 
\begin{align*}
    &B\sum_{i\in S_t}|F(\mathbf{f}(\mathbf{x}_{t,S_t})| \leq BK|F(\mathbf{f}(\mathbf{x}_{t,S_t})|
    \\
    &B\sum_{i\in S_t}|F(\mathbf{f}(\mathbf{x}_{t,S_t})|\leq BK^2 4\sqrt{1/\lambda_1}(9T^{-\beta\alpha+0.5}) n^{L_m}
\end{align*}
Now combining all the results we had, we get the following bound which is a bound on the expected super arm regret of a single round
\begin{align*}
    &u(\mathbf{h}(\mathbf{x}_{t,S_t^*}))-u(\mathbf{h}(\mathbf{x}_{t,S_t})) \leq B\bigg(\sum_{i\in S_t}I_{1,t} + \sum_{i\in S_t}I_{2,t} + \sum_{i\in S_t}I_{3,t}\bigg)
    \\
    &  \leq B'BK\left( \frac{C^{''}Q\delta_t \sqrt{d}N}{P_t Y} +  \frac{NQ\delta_t\sqrt{d}}{Y} \right) +  
    \\
    & 2B \gamma_{T}\sqrt{K}\sqrt{{\sum_{i\in S_t}{\|\mathbf{g}(\mathbf{x}_{t,i};\boldsymbol{\theta}_{t-1})/\sqrt{m}\|_{\mathbf{Z}_{t-1}^{-1}}^2}}} +BKe_t \\
    &+  BK^2 4\sqrt{1/\lambda_1}(9T^{-\beta\alpha+0.5}) n^{L_m}
\end{align*}
\end{proof}

% \begin{align*}
%     &J(t) = \sum_{i\in c, c\in \mathcal{A}_t} \|\mathbf{x}_{i,t}-\mathbf{\bar{x}}_{c,t}\|_1\\
%     & J(t) \geq \|\mathbf{x}_{i,t}-\mathbf{\bar{x}}_{c,t}\|_1 \geq \| \mathbf{\bar{x}}_{c,t} \|_1 - \| \mathbf{x}_{i,t}\|_1\\
%     & \| \mathbf{\bar{x}}_{c,t} \|_1 \leq J(t)+\| \mathbf{x}_{i,t} \|_1\\
%     & \|\mathbf{x}_{c^{*}_t} \|_1 -\|\mathbf{\bar{x}}_{c,t} \|_1 \leq \|\mathbf{x}_{c^{*}_t}-\mathbf{\bar{x}}_{c,t}\|_1 \leq P_t\\ &\|\mathbf{x}_{c^{*}_t} \|_1 \leq P_t + \|\mathbf{\bar{x}}_{c,t}\|_1 \leq P_t + J(t) + \|\mathbf{x}_{i,t}\|_1 \\
%     %& %\|\mathbf{x}_{c^{*}_t} \|_1 \leq P_t + J(t) + dVN
% \end{align*}

\addtocounter{thm}{-1}
\begin{thm} 
Supposing that all the assumptions \ref{continuity_h}-\ref{ell0_assump} hold, let $\mathbf{h}=[h(\mathbf{x}^i)]_{i=1}^{TN}]$, and select the hyperparameters of the algorithm as follows %\carlee{add $\delta_t$ assumption} \com{Done}
\begin{align*}
    &\eta_1 = C_1(T K m L + m\lambda_1)^{-1}
    \\    &\frac{16m\cdot81T^{(-2\beta\alpha+1)}K^2n^{2L_m}}{e_T^2} \geq\lambda_1\geq C_2LK
    \\
    &J = 2\log\left(\sqrt{\lambda_1/(TK)}/(\lambda_1+C_3TKL)\right)TKL/(C_1\lambda_1) 
    \\
    &S \geq \sqrt{2\mathbf{h}^T\mathbf{H}^{-1}\mathbf{h}}
    \\
\end{align*}
for constants $C_1,C_2,C_3$ defined as in Lemma 3, $C_5 > 0 $ is another constant and $m$ is chosen such that it satisfies Condition 1, then the cumulative expected regret of NeUClust over $T$ rounds is upper bounded by 
\begin{align*}
    \mathcal{R}(T) = \widetilde{O}\left(\widetilde{d}\sqrt{T}\right)
\end{align*}    
\end{thm}
\begin{proof}
From equation \ref{regr_form} we have that
\begin{align*}
    &\mathcal{R}(T) = \sum_{t=1}^{T} ( u(S_t^*,\mathbf{h}(\mathbf{x}_{t,S_t^*})))  -u(S_t,\mathbf{h}(\mathbf{x}_{t,S_t}))) 
    \\
    &\mathcal{R}(T)\leq \sum_{t=1}^{T} \Bigg(\underbrace{B'BK\left( \frac{C^{''}Q\delta_t \sqrt{d}N}{P_t Y} +  \frac{NQ\delta_t\sqrt{d}}{Y} \right)}_{\text{$E_{t}$}}+\\
    &\underbrace{2B \gamma_{T}\sqrt{K}\sqrt{{\sum_{i\in S_t}{\|\mathbf{g}(\mathbf{x}_{t,i};\boldsymbol{\theta}_{t-1})/\sqrt{m}\|_{\mathbf{Z}_{t-1}^{-1}}^2}}}}_{\text{$Y_{t}$}} + 
    \\    &\underbrace{BKe_t}_{\text{$U_{t}$}} +  \underbrace{BK^2 4\sqrt{1/\lambda_1}(9T^{-\beta\alpha+0.5}) n^{L_m}}_{\text{$P$}} \Bigg)
\end{align*}
By Lemma 3, we get the above inequality, now focusing on $E_{t}$ we can see that 
\begin{align*}      &\sum_{t=1}^TE_{t}\leq B'BK\sum_{t=1}^T\left(\frac{C^{''}Q\delta_t\sqrt{d}N}{P_tY}+\frac{NQ\delta_t\sqrt{d}}{Y}\right)\\
& \leq B'BKC_5\left(\frac{C^{''}Q\sqrt{d}N}{Y}\sum_{t=1}^T \frac{1}{C_6t^{1.5}} +\frac{NQ\sqrt{d}}{Y}\sum_{t=1}^T\frac{1}{t}\right)\\
&\leq B'BKC_5\left(\frac{C^{''}Q\sqrt{d}N\left(2\sqrt{T}-1\right)}{C_6Y}+\frac{NQ\sqrt{d}(2\sqrt{T}-1)}{Y}\right) 
\end{align*}
where the first inequality follows by Lemma 3, the second inequality follows by the selection of $\delta_t = \frac{C_5}{t}$ for some $C_5> 0$ and $P_t = C_6\sqrt{t}$. The final inequality is due to the fact that $\forall T\geq 2, \sum_{t=1}^T\frac{1}{t} < 2\sqrt{T}-1$ and since $\sum_{t=1}^T \frac{1}{t^{1.5}} < \sum_{t=1}^T \frac{1}{t}$. Selecting $\delta_t$ in such a way corresponds to assuming that our clustering scheme finds the correct clustering as the round number increases. We want to note that the selection of $P_t$ is not necessary and we could instead let $P_t = C_7$ for some $C_7>0$ and the order of the bound would still be the same. For the $Y_{t}$ term we have
%\carlee{should be $Y_{1,t}$ since it depends on time?}\com{Yes, fixed}  
\begin{align*}
    &\sum_{t=1}^{T}Y_{t} = 2B\gamma_T\sqrt{K} \sum_{t=1}^{T}{{\sqrt{{\sum_{i\in S_t}{\|\mathbf{g}(\mathbf{x}_{t,i};\boldsymbol{\theta}_{t-1})/\sqrt{m}\|_{\mathbf{Z}_{t-1}^{-1}}^2}}}}} \\
    & \sum_{t=1}^{T}Y_{t} \leq 2B\gamma_T\sqrt{K} \sqrt{{{T\sum_{t=1}^{T}\sum_{i\in S_t}{(\|\mathbf{g}(\mathbf{x}_{t,i};\boldsymbol{\theta}_{t-1})/\sqrt{m}\|_{\mathbf{Z}_{t-1}^{-1}})^2}}}} \\
    & \leq 2B\gamma_T\sqrt{K}\sqrt{T(2\widetilde{d}\log(1+T N/\lambda_1) + 2 + X)}
    \\
    & \leq 2B\gamma_T\sqrt{K}\sqrt{T(2\widetilde{d}\log(1+T N/\lambda_1) + 3)}
\end{align*}
where the first inequality holds due to the Cauchy-Schwarz inequality, the second inequality holds due to Lemma 11 where $X=C_3 T^{\frac{5}{3}} K^{\frac{5}{3}} L^{4} \lambda_1^{-\frac{1}{6}} m^{-\frac{1}{6}} \sqrt{\log m}$ and the last inequality holds for sufficiently large $m$. By the analysis of \cite{pmlr-v202-hwang23a} in their Theorem 1 and Condition 1, we can bound the scaling factor 
\begin{align*}
    \gamma_T \leq 2\zeta\sqrt{\widetilde{d}\log(1+TN/\lambda_1)+3-2\log\Delta}+2\sqrt{\lambda_1 }S+2
\end{align*}
which happens for the selected $\eta_1$ and $J$ values stated above in the theorem description. Now focusing on the remaining terms, we have 
\begin{align*}
    &\sum_{t=1}^{T}U_{t} = TBKe_T \leq BK(\gamma_T+1) 
    \\
    &\sum_{t=1}^{T}U_{t} \leq BK\bigg(2\zeta\sqrt{\widetilde{d}\log(1+TN/\lambda_1)+3-2\log\Delta}+\\
    &2\sqrt{\lambda_1}S+3\bigg)
\end{align*}
where the first inequality holds because $e_t\leq e_T$ and the rest hold for sufficiently large $m$. For the $P$ term we have %\carlee{label this term in the first regret inequality, similarly for the clustering term} \com{done} 
\begin{align*}
    \sum_{t=1}^TP \leq BK^2 4\sqrt{1/\lambda_1}(9T^{-\beta\alpha+1.5}) n^{L_m}
\end{align*}
where we can choose $\beta\alpha = 1$ so that 
\begin{align*}
    \sum_{t=1}^TP \leq BK^2 4\sqrt{1/\lambda_1}(9\sqrt{T}) n^{L_m}
\end{align*}
 Now combining all the terms together we get the final regret as 
\begin{align*}
    &\mathcal{R}(T) < 2B\sqrt{K}\sqrt{T(2\widetilde{d}\log(1+T N/\lambda_1) + 3)}\cdot
    \\    &\left(2\zeta\sqrt{\widetilde{d}\log(1+TN/\lambda_1)+3-2\log\Delta}+2\sqrt{\lambda_1 }S+2\right)+
    \\    &BK\left(2\zeta\sqrt{\widetilde{d}\log(1+TN/\lambda_1)+3-2\log\Delta}+2\sqrt{\lambda_1}S+3\right)+
    \\
    &BK^2 4\sqrt{1/\lambda_1}(9\sqrt{T}) n^{L_m}+\\
    & B'BKC_5\left(\frac{C^{''}Q\sqrt{d}N\left(2\sqrt{T}-1\right)}{C_6Y}+\frac{NQ\sqrt{d}(2\sqrt{T}-1)}{Y}\right) 
\end{align*}
Thus, we obtain $\widetilde{O}\left(\widetilde{d}\sqrt{T}\right)$ expected regret over $T$ rounds.
\end{proof}
\subsection{Theoretical Results from Other Works}
\begin{lem}
\label{lemma 5.1}(Lemma 5.1 in \cite{neuralucbzhou20a}) For any $\Delta \in (0, 1)$, suppose that there exists a positive constant $C$ such that
\begin{equation*}
    m \geq CT^4 N^4 L^6 \lambda_0^{-4} \log(T^2 N^2 L/\Delta).
\end{equation*}
Then, with probability at least $1 - \Delta$, there exists a $\boldsymbol{\theta}^* \in \mathbb{R}^p$ such that for all $i \in [TN]$,
\begin{equation*}
    h(\mathbf{x}^i) = \mathbf{g}(\mathbf{x}^i; \boldsymbol{\theta}_0)^T (\boldsymbol{\theta}^* - \boldsymbol{\theta}_0) ,
\end{equation*}
\begin{equation*}
    \sqrt{m} \|\boldsymbol{\theta}^* - \boldsymbol{\theta}_0\|_2 \leq \sqrt{\mathbf{h}^T \mathbf{H}^{-1} \mathbf{h}},
\end{equation*}
where $\mathbf{H}$ is the NTK matrix defined in Definition 1 and $\mathbf{h} = [h(\mathbf{x}^i)]_{i=1}^{TN}$.
\end{lem} 
%\subsubsection{\textbf{Lemma 5}} 
\begin{lem}
\label{lemma 4_1}(Lemma 4.1 in \cite{cao2019}) Suppose that there exist $\Bar{C}_1, \Bar{C}_2,\Bar{C}_3 > 0$ such that for any $\Delta \in (0, 1)$, $\tau$ satisfies
\begin{equation*}
    \Bar{C}_1 m^{-\frac{3}{2}} L^{-\frac{3}{2}} \left(\log(TN L^2/\Delta)\right)^{\frac{3}{2}} \leq \tau \leq \Bar{C}_2 L^{-6} \left(\log m\right)^{-\frac{3}{2}}.
\end{equation*}
    Then, with probability at least $1 - \Delta$, for all $\widetilde{\boldsymbol{\theta}},\hat{\boldsymbol{\theta}}$ satisfying $\|\widetilde{\boldsymbol{\theta}} - \boldsymbol{\theta}_0\|_2 \leq \tau$, $\|\hat{\boldsymbol{\theta}} - \boldsymbol{\theta}_0\|_2 \leq \tau$ and $i \in [TN]$, we have
\begin{equation*}
    \left| f(\mathbf{x}^i; \widetilde{\boldsymbol{\theta}}) - f(\mathbf{x}^i; \hat{\boldsymbol{\theta}}) - g(\mathbf{x}^i; \hat{\boldsymbol{\theta}})^T (\widetilde{\boldsymbol{\theta}} - \hat{\boldsymbol{\theta}}) \right| \leq \Bar{C}_3 \tau^{\frac{4}{3}} L^{3} \sqrt{m \log m}.
\end{equation*}
\end{lem}
%\subsubsection{\textbf{Lemma 6}} 
\begin{lem}
\label{lemma B_3} (Lemma B.3 in \cite{cao2019}) There exist constants $\{\Bar{C}_j\}_{j=1}^3 > 0$ such that for any $\Delta > 0$, if $\tau$ satisfies that
\begin{equation*}
    \Bar{C}_1 m^{-\frac{3}{2}} L^{-\frac{3}{2}} \left(\log(TN L^2/\Delta)\right)^{\frac{3}{2}} \leq \tau \leq \Bar{C}_2 L^{-6} \left(\log m\right)^{-\frac{3}{2}},
\end{equation*}
then with probability at least $1 - \Delta$, for any $\|\boldsymbol{\theta} - \boldsymbol{\theta}_0\|_2 \leq \tau$ and $i \in [TN]$ we have
\begin{equation*}
    \|\mathbf{g}(\mathbf{x}^i; \boldsymbol{\theta})\|_F \leq \Bar{C}_3 \sqrt{mL}.
\end{equation*}
\end{lem}
%\subsubsection{\textbf{Lemma 7}} 
\begin{lem}
\label{lemma 7_1} (Lemma 7.1 in \cite{pmlr-v97-allen-zhu19a}) If $\varepsilon \in (0,1]$, with probability at least $1 - O(nL) \cdot e^{-\Omega(m\varepsilon^2/L)}$ over the randomness of $\boldsymbol{\theta}_0 = [$vec$(\mathbf{W}_0)^T,...,$vec$(\mathbf{W}_{L})^T]^T$, we have
\[
\forall \ell \in \{0,\ldots,L\} : \|\mathbf{s}_{\ell}\|_2 \in [1 - \varepsilon, 1 + \varepsilon].
\]
where $\mathbf{s}$ denotes the output of the hidden layer.
\end{lem}
%\subsubsection{\textbf{Lemma 8}} 
\begin{lem}
\label{lemma 8_2}(Lemma 8.2 in \cite{pmlr-v97-allen-zhu19a}) Suppose \(\omega \leq \frac{1}{CL^\frac{9}{2} \log^{3} m}\) for some sufficiently large constant \(C > 1\). With probability at least \(1 - e^{-\Omega(m\omega^\frac{2}{3} L)}\), for every \(\overrightarrow{\mathbf{W}}'\) satisfying \(\|\overrightarrow{\mathbf{W}}'\|_2 \leq \omega\), $\|\mathbf{s}_{\ell}- \mathbf{s}_{\ell}^{(0)}\|_2 \leq O(\omega L^{\frac{5}{2}}\sqrt{\log m})$ where $\overrightarrow{\mathbf{W}}'$ 
represents a perturbation to the weights of the network.
\end{lem}
%\subsubsection{\textbf{Lemma 9}} 
\begin{lem}
\label{lemma B_2}(Lemma B.2 in \cite{neuralucbzhou20a})
% \begin{align*}
%     &\mathbf{\Bar{Z}}_t = \lambda \mathbf{I} + \sum_{i=1}^t \mathbf{g}(\mathbf{x}_{i,a_i};\boldsymbol{\theta}_0)\mathbf{g}(\mathbf{x}_{i,a_i};\boldsymbol{\theta}_0)^T/m,\\  
%     &\mathbf{\Bar{b}}_t = \sum_{i=1}^t r_{i,a_i} \mathbf{g}(\mathbf{x}_{i,a_i};\boldsymbol{\theta}_0)/\sqrt{m},\\
%     &\Bar{\gamma}_t = \nu \sqrt{\log\left(\frac{\det \mathbf{\Bar{Z}}_t}{\det \lambda \mathbf{I}} \right)- 2\log \delta }+ \sqrt{\lambda}S.
% \end{align*}
There exist constants $\{\Bar{C}_i\}_{i=1}^5 > 0$ such that for any $\Delta > 0$, if for all $t \in [T]$, $\eta_1, m$ satisfy
\begin{align*}
    &{2\sqrt{t/(m\lambda_1)}}\geq \Bar{C}_1 m^{-3/2} L^{-3/2}\left(\log(TN L^2/\Delta)\right)^{3/2}
    \\    &2\sqrt{t/(m\lambda_1)}\leq \Bar{C}_2 \min \big\{ L^{-6}(\log m)^{-3/2},\\    
    &(m(\lambda_1\eta_1)^2 L^{-6} t^{-1} (\log m)^{-1})^{3/8} \big\}
    \\
    &\eta_1 \leq \Bar{C}_3 (m\lambda_1 + t m L)^{-1}
    \\
    &m^{1/6} \geq \Bar{C}_4 \sqrt{\log m} L^{7/2} t^{7/6}\lambda_1^{-7/6}(1 + \sqrt{t/\lambda_1})
\end{align*}
then with probability at least $1 - \Delta$, we have that $\|\boldsymbol{\theta}_t - \boldsymbol{\theta}_0\|_2 \leq 2\sqrt{t/(m\lambda_1)}$
\end{lem}
%\subsubsection{\textbf{Lemma 10.}} 
\begin{lem}
\label{lemma 1}(Lemma 1 in \cite{pmlr-v202-hwang23a}) For any \(\Delta \in (0,1)\), suppose the width of the neural network \(m\) satisfies Condition 1. Let \(\gamma_t\) be a positive scaling factor defined as
\begin{align*}
   \gamma_t = \Gamma_{1,t} \left(\zeta\sqrt{\log\frac{\det \mathbf{Z}_t}{\det \lambda \mathbf{I}} + \Gamma_{2,t} - 2\log \delta} + \sqrt{\lambda_1 S}\right) + \\
    (\lambda_1 + C_1 t K L) \left((1 - \eta_1 m \lambda_1)^{\frac{J}{2}} \sqrt{\frac{tK}{\lambda_1}} + \Gamma_{3,t} \right), 
\end{align*}

where
\[
\begin{aligned}
\Gamma_{1,t} &= \sqrt{1 + C_{\Gamma,1} t^{\frac{7}{6}} K^{\frac{7}{6}} L^{4} \lambda_1^{-\frac{7}{6}} m^{-\frac{1}{6}} \sqrt{\log m}}, \\
\Gamma_{2,t} &= C_{\Gamma,2} t^{\frac{5}{3}} K^{\frac{5}{3}} L^{4} \lambda_1^{-\frac{1}{6}} m^{-\frac{1}{6}} \sqrt{\log m}, \\
\Gamma_{3,t} &= C_{\Gamma,3} t^{\frac{7}{6}} K^{\frac{7}{6}} L^{\frac{7}{2}} \lambda_1^{-\frac{7}{6}} m^{-\frac{1}{6}} \sqrt{\log m}(1 + \sqrt{tK/\lambda_1}),
\end{aligned}
\]
for some constants $C_1, C_{\Gamma,1}, C_{\Gamma,2}, C_{\Gamma,3} > 0$. If $\eta_1 \leq C_2(T K m L + m \lambda_1)^{-1}$ for some $C_2 > 0$, then for any $t \in [T]$ and $i \in [N]$, with probability at least $1 - \Delta$ we have
\[
|u_{t,i} - h(\mathbf{x}_{t,i})| \leq 2 \gamma_{t-1}\|\mathbf{g}(\mathbf{x}_{t,i};\boldsymbol{\theta}_{t-1})/\sqrt{m}\|_{\mathbf{Z}_{t-1}^{-1}} + e_t ,
\]
where
$u_{t,i} = f(\mathbf{x}_{t,i};\boldsymbol{\theta}_{t-1}) + \gamma_t || \mathbf{g}(\mathbf{x}_{t,i};\boldsymbol{\theta}_{t-1})/ \sqrt{m} || _{\mathbf{Z}_{t-1}^{-1}}$ and \(e_t\) is defined for some absolute constants \(C_3, C_4 > 0\) as follows.
\begin{align*}
    &e_t := C_3 \gamma_{t-1} t^{\frac{1}{6}} K^{\frac{1}{6}} L^{\frac{7}{2}} \lambda_1^{-\frac{2}{3}} m^{-\frac{1}{6}} \sqrt{\log m} \\
    &+ C_4 t^{\frac{2}{3}} K^{\frac{2}{3}} \lambda_1^{-\frac{2}{3}} m^{-\frac{1}{6}} \sqrt{\log m}.
\end{align*}

\end{lem}

%\subsubsection{\textbf{Lemma 11}} 
\begin{lem}
\label{Lemma 2} (Lemma 2 in \cite{pmlr-v202-hwang23a})

For any \(\Delta \in (0,1)\), suppose the width of the neural network \(m\) satisfies Condition 1. If \(\eta_1   \leq C_1(T K m L + m\lambda_1)^{-1}\), and \(\lambda_1 \geq C_2 L K\), for positive constants \(C_1, C_2\) with \(C_2 \geq \sqrt{\max_{t,i} \left\|\mathbf{g}(\mathbf{x}_{t,i};\boldsymbol{\theta}_{t-1})/\sqrt{m}\right\|_2^2/L}\), then with probability at least \(1 - \Delta\), for some \(C_3 > 0\),
\begin{align*}
   &\sum_{t=1}^{T} \sum_{i \in S_t} \left\|\mathbf{g}(\mathbf{x}_{t,i};\boldsymbol{\theta}_{t-1})/\sqrt{m}\right\|^2_{\mathbf{Z}_{t-1}^{-1}} 
   \\
   &\leq 2\widetilde{d}\log(1+T N/\lambda_1) + 2 + C_3 T^{\frac{5}{3}} K^{\frac{5}{3}} L^{4} \lambda_1^{-\frac{1}{6}} m^{-\frac{1}{6}} \sqrt{\log m}. 
\end{align*}
\end{lem}

\begin{cor}
%\subsubsection{\textbf{Corollary 1}} 
\label{Corollary_1}(Corollary 1 in \cite{pmlr-v202-hwang23a}). With probability at least $1 - \Delta$
\begin{equation*}
    u_{t,i} + e_t \geq h(x_{t,i}) . 
\end{equation*}
where $u_{t,i} = f(\mathbf{x}_{t,i};\boldsymbol{\theta}_{t-1}) + \gamma_t || \mathbf{g}(\mathbf{x}_{t,i};\boldsymbol{\theta}_{t-1})/ \sqrt{m} || _{\mathbf{Z}_{t-1}^{-1}}$ and

$e_t := C_3\gamma_{t-1} t^{\frac{1}{6}} K^{\frac{1}{6}} L^{\frac{7}{2}} m^{-\frac{1}{6}}\lambda_1^{-\frac{2}{3}} \sqrt{\log m}
+ C_4 t^{\frac{2}{3}} K^{\frac{2}{3}} m^{-\frac{1}{6}}\lambda_1^{-\frac{2}{3}} \sqrt{\log m}.$
\end{cor}
\section{Further Experimental Details and Results} \label{sec:appendix_exp}
In this section, we first give some details about the MovieLens dataset and then present the explanation of the restaurant recommendation setup for the Yelp dataset. We also explain the offline version of our algorithm NeUClust whose pseudocode is given in Algorithm 2, which is the version which was used to run the experiments on the MovieLens and Yelp datasets. Next, we will explain the experimental details and hyperparameter tuning done for the algorithms. The experiments were done with a time horizon $T = 1000$ and $K=5$. Finally, we give some more experimental results for MovieLens and Yelp and clustering plots for the context space. 

\subsection{Details about MovieLens Dataset}
The MovieLens 25M dataset is composed of 25 million ratings from approximately 162,000 users for 62,000 movies. Each movie has genre data associated with it, where a movie can belong to one or many of the 20 genres listed such as action, comedy, horror etc. Each rating is uniquely identified by a user id and movie id pair and every rating given to a movie falls between 0.5 to 5 with 0.5 increments. As in \cite{elahi2023contextual}, we only consider ratings given to movies after 2015 and we only keep the users who have rated 200 movies or more which corresponds to approximately 5 million ratings, 57,000 movies and 10,000 users. 

\subsection{Restaurant Recommendation Setup }
The Yelp Open dataset comprises of nearly 7 million reviews, 2 million users and 150,000 businesses. Each business has categories data, which describes/specifies some characteristics of the business such as restaurant, mexican, burgers etc. In each round, there is an incoming user, and we need to select $K$ restaurants to recommend to the user from a total of approximately 50,000 restaurants. We use term frequency-inverse document frequency (tf-idf) to construct the contexts for the restaurants and cluster the restaurants using this category information. We filter some categories which do not appear often for a final context dimension of $d = 171$. We also construct the user preference vector using the ratings the user gave for the restaurants having these categories. We construct the mean of a base arm, base arm rewards and super arm rewards the same way as explained in Movie Recommendation Setup now using the restaurant context and user preference vector of the incoming user at round $t$.

\begin{algorithm}
\caption{NeUClust (Offline)}
\begin{algorithmic}[1]
	\STATE\textbf{Input: }{Number of rounds $T$, regularization parameters $\lambda_1$ and $\lambda_2$, step sizes $\eta_1$ and $\eta_2$, number of gradient descent steps $J$, network widths $m$ and $n$, network depths $L$ and $L_m$, number of clusters $M$, maximum number of iterations of clustering $i_c$, size of super arm $K$, norm parameter $S$.}
	\STATE Randomly initialize $\boldsymbol{\theta}_0$ and $\boldsymbol{\Theta}_0$ as described previously and $\mathbf{Z}_0 = \lambda_1 \mathbf{I}$ 
    \STATE Run $k$-means clustering on the observed contexts $\{\mathbf{x}_{i}\}_{i \in [N] }$ with $k=M$ and maximum number of iterations $i_c$

	\FOR{$t = 1, \ldots, T$}
	%\STATE Observe $\mathbf{a}_{z_t}$ (genre vector of movie or user restaurant preference vector)
        \STATE Observe $\{\mathbf{x}_{t,i}\}_{i \in [N] }$
        %\STATE Define $\{\mathbf{x}'_{t,i}\}_{i\in [N]}= \{\mathbf{x}_i \odot \mathbf{a}_{z_t}\}_{i\in[N]} $
    
    \FOR{$c = 1, \ldots, M$}
        \FOR{$j \in c $}    
    
        \STATE $v_{t,c_j} = f(\mathbf{x}_{t,c_j};\boldsymbol{\theta}_{t-1}) + \gamma_t || \mathbf{g}(\mathbf{x}_{t,c_j};\boldsymbol{\theta}_{t-1})/ \sqrt{m} || _{\mathbf{Z}_{t-1}^{-1}}$ 
        \ENDFOR
        
        \STATE Within cluster $c$, find top $K$ cluster elements with the highest $v$ value, $\mathbf{b}_c \in \mathbb{R}^K$ vector includes the indices of these elements 
        \STATE Initialize $\mathbf{V}_t = (0,0,...,0) \in \mathbb{R}^{M}$
        \FOR{$a \in \mathbf{b}_c$}
        \STATE $V_{t,c} \pluseq v_{t,c_a} + F(\mathbf{f}(\mathbf{x}_{t,\mathbf{b}_c};\boldsymbol{\theta}_{t-1}))$
        \ENDFOR
        
        \ENDFOR
        \STATE $c' = \argmax_c(\mathbf{V}_{t})$
        \STATE $S_t = \mathbf{b}_{c'} $

	\STATE Play super arm $S_t$ and observe base arm rewards $\{r_{t,i}\}_{i \in S_t}$ and super arm reward $R_t$

        %\STATE Update context vector $\{\mathbf{x}_{t,i}\}_{i \in S_t }$ of the arms played by using the observed base arm rewards 
        
	\STATE Update $\mathbf{Z}_t = \mathbf{Z}_{t-1} + \sum_{i \in S_t}\mathbf{g}(\mathbf{x}_{t,i};\boldsymbol{\theta}_{t-1})\mathbf{g}(\mathbf{x}_{t,i};\boldsymbol{\theta}_{t-1})^T/m$
        \STATE Update weights $\boldsymbol{\theta}_t$ by minimizing the loss $\mathcal{L}(\boldsymbol{\theta})$ using gradient descent with step size $\eta_1$ for $J$ iterations.
        \STATE Update weights $\boldsymbol{\Theta}_t$ by minimizing the loss $\mathcal{L}(\boldsymbol{\Theta})$ using gradient descent with step size $\eta_2$ for $J$ iterations.
        \STATE Update $\gamma_t$ as explained in Eq. \ref{eq_8} %\carlee{cite equation number}
	\ENDFOR

\end{algorithmic}
\label{Algortihm 2}
\end{algorithm}

\subsection{Baselines, Hyperparameters and Results}

The experiments are run with a GTX 1080 TI GPU model with 11 GB memory on a Linux operating system. For CC-MAB algorithm we use the implementation of \cite{pmlr-v108-nika20a} available in \url{https://github.com/Bilkent-CYBORG/ACC-UCB}. For $K$-LinUCB we make use of an exact oracle to make super arm selections and tune $\alpha$ using grid search between $0.01-1$. For NeuralMAB, we tune the learning rate using grid search between $0.001-0.1$ and use 50 epochs. For CN-UCB and NeUClust we do grid search for the learning rate $\eta_1$ over $0.001-0.1$ and use $\eta_1 = 0.001$ for both. We also tune $\eta_2$ over $0.001-0.1$ for NeUClust and use $\eta_2 = 0.001$ . We also do grid search for $\gamma$ over $0.1-1$ and use $\gamma=1$. We use a regularization parameter $\lambda_1 = 1$ for both NeUClust and CN-UCB and $\lambda_2=1$ for NeUClust. To estimate the base arm reward function, we use a network with a single hidden layer for a total of $p=420$ parameters for MovieLens dataset for NeUClust and CN-UCB and $p=17200$ parameters for Yelp dataset for NeUClust and CN-UCB. We train both NeUClust and CN-UCB for $J=40$ iterations for both datasets using stochastic gradient descent. For the super arm network for NeUClust, we use a one hidden layer network with $90$ parameters for MovieLens and $180$ parameters for Yelp. The reported results in Figures 2 and 3 for NeUClust are for 22 clusters for MovieLens and 20 clusters for Yelp. We run $k-$means clustering for $i_c = 300$.

Now we will present the super arm reward results for MovieLens and Yelp for the discussed algorithms. 

\begin{figure}[h]
\centering
\includegraphics[width=\linewidth]{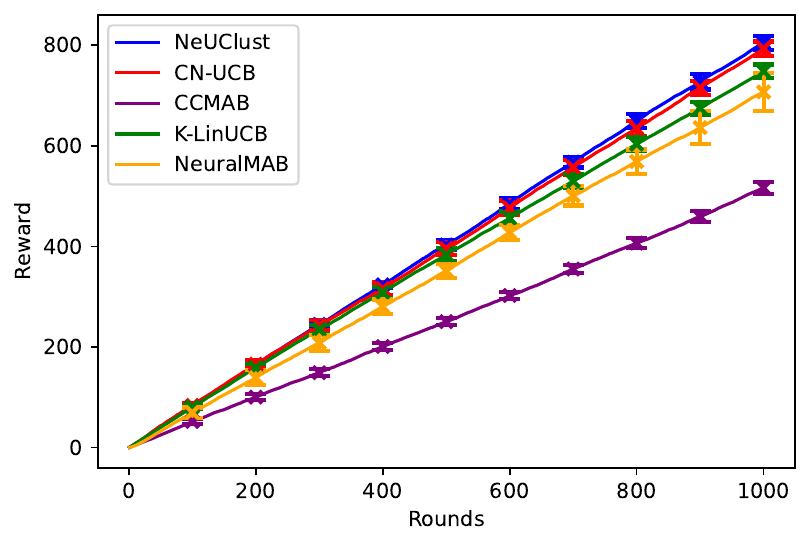}
\caption{Super arm reward vs rounds $(t)$ of NeUClust compared with other state-of-the art combinatorial and neural bandit algorithms for MovieLens}

\end{figure}

\begin{figure}[h]
\centering
\includegraphics[width=\linewidth]{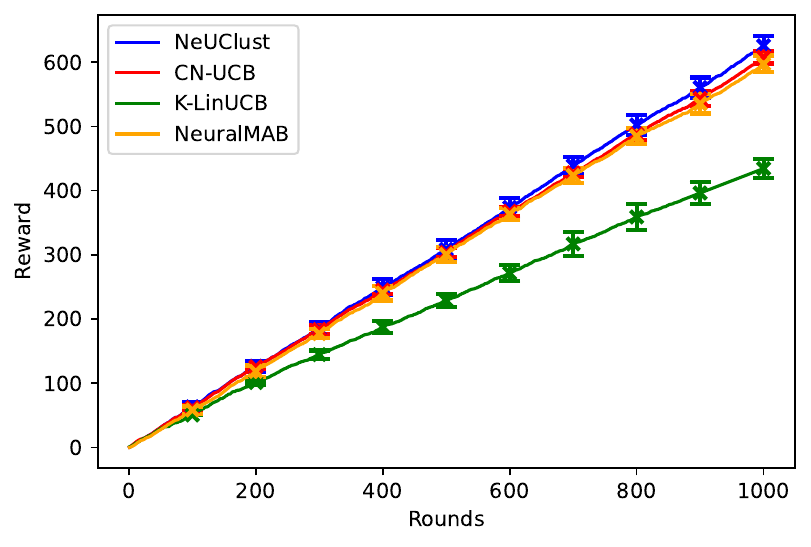}
\caption{Super arm reward vs rounds $(t)$ of NeUClust compared with other state-of-the art combinatorial and neural bandit algorithms for Yelp}

\end{figure}

We can see through Figures 4 and 5 that NeUClust also outperforms the other baseline algorithms in terms of super arm reward and this difference is more noticeable for the Yelp dataset. Similarly to the super arm regret plots, we plot the mean values of 5 independent runs for Yelp and 10 independent runs for MovieLens for all the algorithms and the error bars indicate $\pm$ std.

\begin{figure}[h]
\centering
\includegraphics[width=\linewidth]{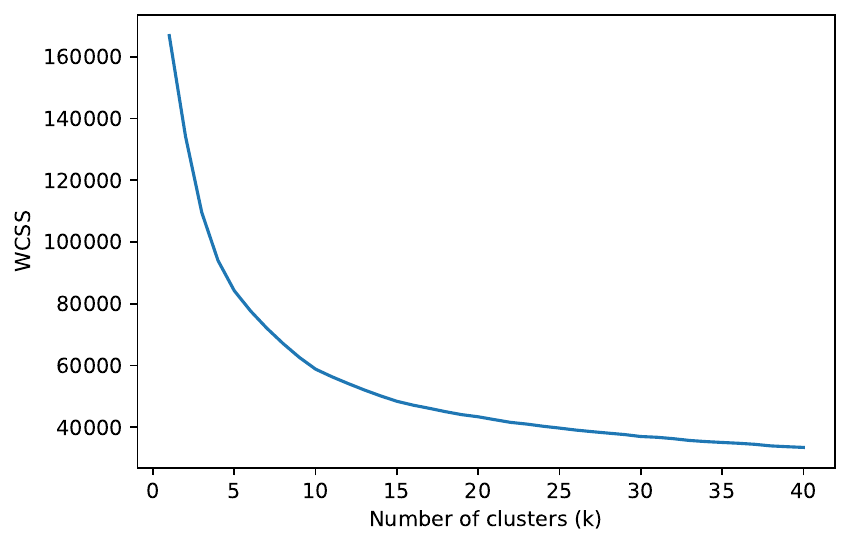}
\caption{Within Cluster Sum of Squares (WCSS) vs number of clusters plot for MovieLens dataset}

\end{figure}

\begin{figure}[h]
\centering
\includegraphics[width=\linewidth]{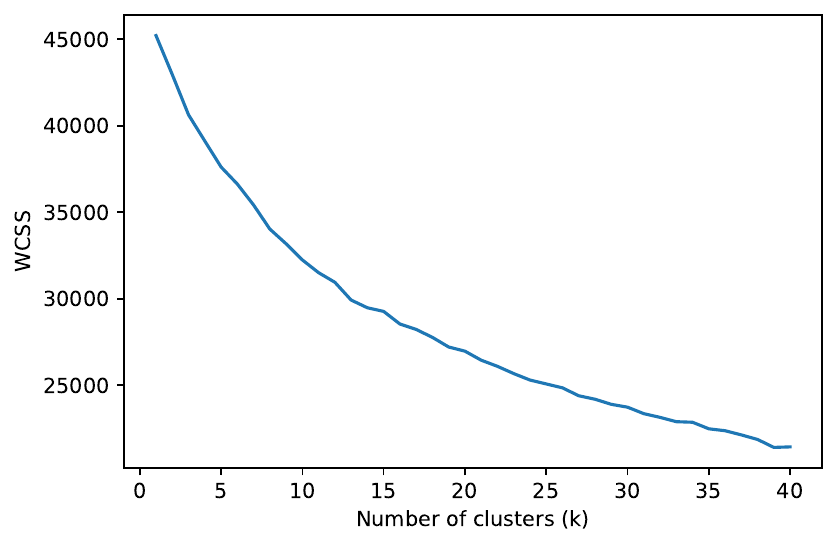}
\caption{Within Cluster Sum of Squares (WCSS) vs number of clusters plot for Yelp dataset}

\end{figure}

Now we focus on the clustering structure of the datasets, in particular we look at the Within Cluster Sum of Squares (WCSS) variation with the number of clusters to produce elbow plots which is a method to understand the optimal range for the number of clusters in the data. For the MovieLens dataset, as mentioned before we cluster over the genre preference vector for the users, which is a vector showing the average rating each user gave for the 20 movie genres. For the Yelp dataset, we instead cluster over the restaurant category data, to identify similar restaurants over 171 categories. It can be seen that the WCSS for MovieLens is much higher compared to Yelp, which is a bit surprising, given that the cluster dimension for Yelp is much larger than for MovieLens (171 vs 20). However, since we do some filtering for the category data and since the rating values are on average much larger than tf-idf values for the categories and also factoring in that there might be some outlier users for movie recommendation which do not specifically align with any of the clusters, these may be some of the main reasons for the main difference in the WCSS values for the two plots. Through Figures 6 and 7 we realize the optimal number of clusters for MovieLens lies somewhere in the range of 10-20 whereas for Yelp it seems to be somewhere in the 15-30 range. 

We will now present some results regarding the variation of the super arm regret and super arm reward for the Yelp and MovieLens datasets as we vary the number of clusters we use for NeUClust. We present the mean values of 5 independent runs for all $k$ values plotted. It should be noted that we made use of these plots to determine the value of $k$ we use to compare with the other algorithms in Figures 2,3,4,5 and to get a general intuition whether our observations from the elbow plots in Figures 6 and 7 are consistent with the results we get in this part.

\begin{figure}[t]
\centering
\includegraphics[width=\linewidth]{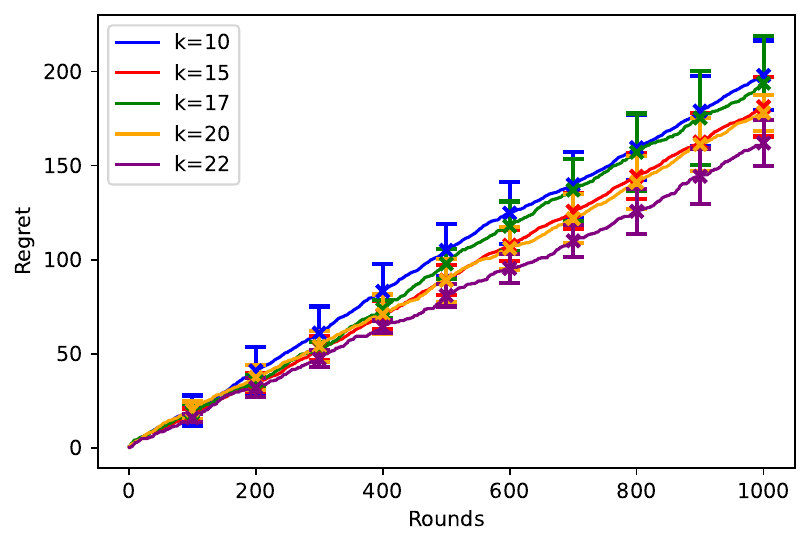}
\caption{Super arm regret vs rounds $(t)$ of NeUClust for different $k$ values for MovieLens}

\end{figure}

\begin{figure}[t]
\centering
\includegraphics[width=\linewidth]{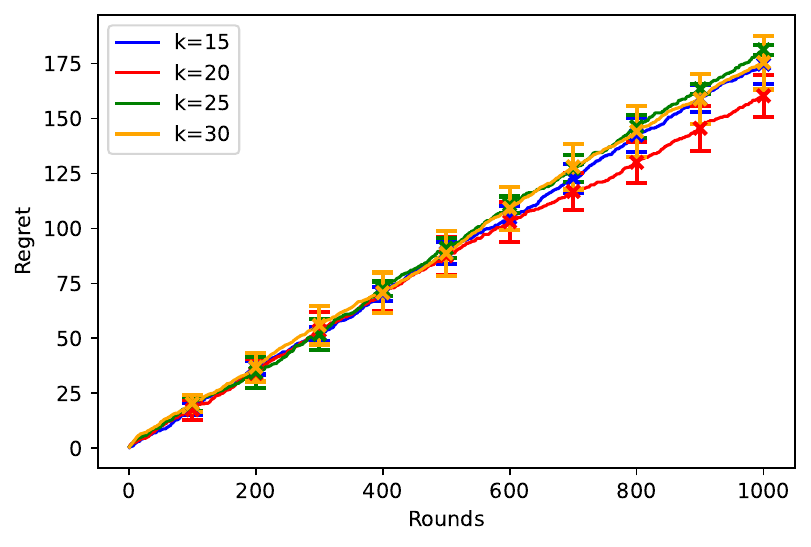}
\caption{Super arm regret vs rounds $(t)$ of NeUClust for different $k$ values for Yelp}

\end{figure}

These plots show that the optimal value of $k$ for MovieLens is around 22 clusters whereas for Yelp it is around 20 clusters. Even though the differences may seem quite small in some of the plots, they still display some trend and indicate how low number of clusters performs poorly overall which verifies our expectation from the elbow plots in Figures 6 and 7. Also it can be said that the regret plots give room for more understanding of the trends since the scale of the regret axis is lower compared to the reward axis.

\end{document}